\documentclass{article} 
\usepackage{arxiv_preprint,times}

\usepackage[utf8]{inputenc} 
\usepackage[T1]{fontenc}    
\usepackage{hyperref}       
\usepackage{url}            
\usepackage{booktabs}       
\usepackage{amsfonts}       
\usepackage{nicefrac}       
\usepackage{microtype}      
\usepackage{xcolor}         
\usepackage{amsmath}
\usepackage{amssymb}
\usepackage{amsthm}
\usepackage{enumitem}
\usepackage{multirow}
\usepackage{graphicx}
\usepackage{wrapfig}
\usepackage{subcaption}

\newtheorem{theorem}{Theorem}
\newtheorem{lemma}{Lemma}

\newtheorem{corollary}{Corollary}

\theoremstyle{definition}

\newtheorem{assumption}{Assumption}
\newtheorem{remark}{Remark}
\newtheorem{hypothesis}{Hypothesis}

\title{SynCo: Learning Cross-Modal Synergy by Contrasting Interaction Residuals}

\author{Yavuz Yarici \\
OLIVES, CSIP, School of ECE\\
Georgia Institute of Technology\\
Atlanta, GA, USA \\
\texttt{yavuzyarici@gatech.edu} \\
\And
Ghassan AlRegib \\
OLIVES, CSIP, School of ECE\\
Georgia Institute of Technology\\
Atlanta, GA, USA \\
\texttt{alregib@gatech.edu} \\
}

\begin{document}

\begin{titlepage}
\thispagestyle{empty}   
\vspace*{1in}

{\large
\begin{itemize}[leftmargin=2.5cm, align=parleft, labelsep=2cm, itemsep=4ex]

\item[\textbf{Citation}] Y. Yarici and G. AlRegib, ``SynCo: Learning Cross-Modal Synergy by Contrasting Interaction Residuals,'' submitted on September 25, 2026.

\item[\textbf{Review}] Under Review 

\item[\textbf{Codes}] Under Review 

\item[\textbf{Bib}] {\raggedright
    @misc\{yarici2026synco,\\
    title=\{SynCo: Learning Cross-Modal Synergy by Contrasting Interaction Residuals\},\\
    author=\{Yarici, Yavuz and AlRegib, Ghassan\},\\
    note=\{Submitted on September 25, 2026\},\\
    year=\{2026\}\}\par}

\item[\textbf{Contact}] \{yavuzyarici, alregib\}@gatech.edu \\
    \url{https://alregib.ece.gatech.edu/}

\end{itemize}
}

\vfill
\end{titlepage}

\maketitle

\begin{abstract}
Multimodal contrastive learning is a dominant paradigm for learning transferable representations from unlabeled data, but standard objectives primarily capture information that is redundant between modalities. Partial Information Decomposition (PID) shows that task-relevant information in multimodal data decomposes into three components: redundancy shared between modalities, uniqueness specific to each modality, and synergy available only from their joint observation. Recent frameworks extend contrastive learning to capture all three components, yet synergy remains undertrained in practice. We propose SynCo (Synergy Contrastive Learning), a method that directly addresses synergy undertraining through dedicated supervision on an interaction residual. SynCo fits a linear projector to predict the fused representation from independently computed unimodal features, and the resulting interaction residual, which removes the linearly unimodal-predictable component, receives dedicated contrastive supervision at negligible computational cost. On the controlled Trifeature benchmark, SynCo achieves state-of-the-art synergy capture with a $+5.98\%$ gain over the baseline, and on real-world benchmarks from MultiBench, DARai, and MM-IMDb, SynCo consistently outperforms or matches prior methods across diverse modality combinations and task types. The method operates as a plug-in to existing contrastive multimodal frameworks without modifying the underlying fusion architecture and can further improve synergy capture when combined with other methods.
\end{abstract}

\vspace{-2mm}
\section{Introduction}
\vspace{-2mm}


Multimodal learning plays a central role in high-stakes applications such as medical imaging~\citep{li2020domain, prabhushankar2022olives}, autonomous systems~\citep{kokilepersaud2023exploiting}, video understanding~\citep{kaviani2025hierarchical}, and scientific applications~\citep{orlyanchik2025co, massey2025earthscape}. Such domains benefit from leveraging complementary cues across modalities, producing richer representations than unimodal approaches. Complementary signals are particularly valuable in real-world settings where data variation is significant~\citep{Hendrycks_2021_ICCV, yarici2025subject}. However, capturing the cross-modal interactions that produce these representations remains a fundamental challenge in multimodal learning.

Multimodal contrastive learning has emerged as the dominant approach to this challenge. Methods such as CLIP~\citep{radford2021learning} and ALIGN~\citep{jia2021scaling} learn joint representations across modalities using contrastive objectives~\citep{oord2018representation, chen2020simple}. Contrastive methods achieve strong transfer across downstream tasks. However, these methods are built on the multi-view redundancy assumption, which holds that all task-relevant information is shared between modalities~\citep{tsai2021self, tian2020contrastive}. As a result, they capture only the information shared between modalities and miss the complementary modality-specific information~\citep{sridharan2008information}. The assumption is restrictive, since many real-world tasks depend on complementary information unique to individual modalities~\citep{baltruvsaitis2018multimodal, zong2024self}.

Partial Information Decomposition (PID)~\citep{williams2010nonnegative, bertschinger2014quantifying} provides a principled framework for characterizing the limitations of redundancy-based objectives. For two modalities $X_1$ and $X_2$ predicting a task $Y$, PID decomposes the total mutual information $I(X_1, X_2; Y)$ into three components: redundancy ($R$), information shared across modalities; uniqueness ($U_1$, $U_2$), information specific to each modality; and synergy ($S$), information that emerges only from their joint observation. Each component carries distinct task-relevant information, and capturing all components is necessary for learning representations that reflect the full structure of multimodal information.

Despite recent advances in multimodal contrastive learning, effectively capturing synergistic information remains challenging. FactorCL~\citep{liang2023factorized} explicitly separates shared and unique representations but relies on restrictive conditional augmentation assumptions and does not explicitly model synergy. CoMM~\citep{dufumier2024align} extends multimodal contrastive learning to capture redundancy, uniqueness, and synergy by maximizing mutual information between augmented multimodal representations alongside unimodal contrastive objectives. While redundancy and uniqueness are accessible from individual modalities, synergy requires their joint observation. Consequently, synergy can remain undertrained even when the contrastive objective theoretically preserves all three components. InfMasking~\citep{wen2025infmasking} addresses this challenge by stochastically masking modality features during fusion and aligning masked representations with unmasked ones through mutual information maximization. COrAL~\citep{cissee2026orthogonalized} promotes synergistic interactions through asymmetric masking while separating shared and modality-specific representations using orthogonality constraints. Both approaches encourage cross-modal interactions through masking-based training, without explicitly separating the linearly unimodal-predictable component from the fused representation.

We propose \textbf{SynCo} (\textbf{Syn}ergy \textbf{Co}ntrastive Learning) to address synergy undertraining in multimodal contrastive learning through dedicated supervision on the interaction residual. SynCo introduces a Synergy Head consisting of a linear projector that predicts the fused representation from unimodal representations and a residual operation that removes the linearly unimodal-predictable component. The resulting interaction residual is used to encourage synergy capture through a dedicated contrastive loss that complements the base objective. SynCo adds only two loss terms and a lightweight prediction head to the base framework, with negligible computational overhead in the evaluated setting and no modifications to the underlying fusion architecture.

\begin{itemize}[leftmargin=*, nosep, topsep=0pt, partopsep=0pt]
    \item We introduce a residual decomposition that removes the linearly unimodal-predictable component of multimodal representations. A linear projector captures the part of the fused representation that is linearly predictable from unimodal features, and the resulting interaction residual provides a direct target for synergy-focused supervision.
    \item We propose SynCo, a method that applies dedicated contrastive supervision on the interaction residual to enhance synergy capture. SynCo operates as a plug-in to existing multimodal contrastive frameworks and adds only two loss terms and a lightweight prediction head at negligible computational overhead.
    \item We demonstrate SynCo's effectiveness across multiple benchmarks. On the Trifeature benchmark, SynCo achieves state-of-the-art synergy capture, exceeding the baseline by $+5.98\%$. On real-world benchmarks from MultiBench, DARai, and MM-IMDb, SynCo consistently outperforms or matches prior methods across vision-language, sensor-based, and clinical domains.
\end{itemize}

\vspace{-1mm}
\section{Related Work}
\label{sec:related_work}
\vspace{-1mm}
\paragraph{Multimodal learning.}
Multimodal learning combines information from heterogeneous sources such as text, images, audio, and tabular data into a single model~\citep{baltruvsaitis2018multimodal, liang2024foundations}. Early methods combined modalities either at the input level through feature concatenation or at the decision level by aggregating unimodal predictions~\citep{snoek2005early, baltruvsaitis2018multimodal}. More recent work uses transformer architectures to model cross-modal interactions directly through attention~\citep{lu2019vilbert, xu2023multimodal}, enabling applications across tasks such as translation~\citep{xu2015show, rombach2022high} and representation learning~\citep{bengio2013representation}. Among these directions, self-supervised representation learning has emerged as a leading approach for learning general-purpose representations from unlabeled data through contrastive objectives~\citep{zong2024self, radford2021learning, jia2021scaling}. Our work targets this setting.
\vspace{-2mm}
\paragraph{Self-supervised multimodal representation learning.}
Self-supervised learning trains models using objectives defined directly on the input data, without requiring human labels~\citep{balestriero2023cookbook}. In the multimodal setting, these methods use cross-modal information to learn transferable representations~\citep{zong2024self}. Prior work can be organized by the type of pretext task. Clustering-based approaches use cluster assignments as pseudo-labels, allowing modalities to supervise one another~\citep{alwassel2020self}. Masked modeling methods reconstruct masked portions of the input, often in a cross-modal fashion where one modality predicts missing information from another~\citep{bachmann2022multimae, mizrahi20234m}. Contrastive methods align representations of matched multimodal pairs while pushing unmatched pairs apart~\citep{radford2021learning, jia2021scaling}. Our work builds on this contrastive line.

\paragraph{Contrastive multimodal interactions.}
Standard contrastive multimodal methods~\citep{radford2021learning, jia2021scaling} optimize cross-modal alignment but primarily capture redundant information shared across modalities. FactorCL~\citep{liang2023factorized} explicitly factorizes shared and unique information, but relies on restrictive conditional augmentation assumptions and does not model synergy. CoMM~\citep{dufumier2024align} grounds its objectives in Partial Information Decomposition~\citep{williams2010nonnegative, bertschinger2014quantifying}, combining multimodal and unimodal contrastive losses to capture redundancy, uniqueness, and synergy. In practice, however, synergy remains the most difficult component to learn, since it requires cross-modal computation. InfMasking~\citep{wen2025infmasking} addresses this by stochastically masking modality features during fusion, which encourages cross-modal computation through input perturbation but does not supervise synergy directly. COrAL~\citep{cissee2026orthogonalized} separates shared and modality-specific features through a dual-path architecture with orthogonality constraints and applies asymmetric masking with complementary patterns across views, which also promotes synergy through input perturbation. In contrast, our method targets synergy explicitly through a dedicated contrastive objective on the interaction residual.

\vspace{-1mm}
\section{Preliminaries}
\vspace{-1mm}

\label{sec:preliminaries}

We introduce the problem setup and Partial Information Decomposition, and review the multimodal contrastive learning framework introduced by CoMM~\citep{dufumier2024align}.

\paragraph{Problem setup.}
Let $X_1, X_2, \ldots, X_n$ be random variables representing $n$ data modalities and let $Y$ denote a target variable for a downstream task. The objective is to learn a multimodal representation $Z = f_\theta(X)$, where $X = (X_1, \ldots, X_n)$ and $\theta$ denotes the parameters of a multimodal encoder, such that $Z$ preserves all task-relevant information: $I(Z; Y) = I(X; Y)$. During pretraining, two augmentations $t'$ and $t''$ are drawn from a set $\mathcal{T}$ of multimodal augmentations, independently of each other and of $(X, Y)$, and applied to the same input. We write $X' = t'(X)$ and $X'' = t''(X)$ for the two views and $Z' = f_\theta(X')$ and $Z'' = f_\theta(X'')$ for their representations. The two views are identically distributed and differ only in the realized augmentation. For each modality $i$, the projection $t_i$ masks all modalities except the $i$-th, and $Z_i = f_\theta(t_i(X))$ denotes the corresponding unimodal representation, while $Z'_i = f_\theta(t_i(X'))$ and $Z''_i = f_\theta(t_i(X''))$ denote its counterparts for the two views. Since $Z$ and every $Z_i$ are outputs of the same encoder $f_\theta$, all representations lie in the same space $\mathbb{R}^D$, where $D$ is the output dimension of $f_\theta$.

\paragraph{Partial Information Decomposition.}
To characterize the structure of task-relevant information, we employ Partial Information Decomposition (PID)~\citep{williams2010nonnegative, bertschinger2014quantifying,liang2023quantifying}. Our theoretical analysis is presented for $n=2$ modalities, since multimodal interactions have not been characterized by PID for larger $n$~\citep{dufumier2024align}. For two modalities $X_1$ and $X_2$ and a target $Y$, PID decomposes the total task-relevant information $I(X_1, X_2; Y)$ into four non-negative terms:
\begin{equation}
    I(X_1, X_2; Y) = R + S + U_1 + U_2,
    \label{eq:pid}
\end{equation}
where $R$ denotes redundancy (information shared across modalities), $U_i$ denotes uniqueness (information specific to modality $i$), and $S$ denotes synergy (information available only from the joint observation of both modalities). This decomposition satisfies the consistency equations:
\begin{equation}
    I(X_1; Y) = R + U_1, \quad I(X_2; Y) = R + U_2, \quad I(X_1; X_2; Y) = R - S,
    \label{eq:consistency}
\end{equation}

\paragraph{Multimodal contrastive learning.}
In self-supervised learning, $Y$ remains unspecified during training, so task-relevant information cannot be optimized directly. CoMM~\citep{dufumier2024align} addresses this by assuming the existence of multimodal augmentations that preserve all task-relevant information, allowing it to be optimized indirectly through contrastive objectives between augmented views.

\begin{assumption}[Minimal label-preserving multimodal augmentations, \citealp{dufumier2024align}]
\label{ass:augmentation}
The set $\mathcal{T}$ of multimodal augmentations preserves task-relevant content: for any $t \in \mathcal{T}$ and $X' = t(X)$, $I(X; X') = I(X; Y)$ and $I(X'; Y) = I(X; Y)$.
\end{assumption}

The assumption states that augmenting $X$ to $X'$ preserves the information $X$ carries about the target $Y$, even though $Y$ itself is not used to construct the augmentation. Unlike standard contrastive assumptions, the preserved information is not required to be shared across modalities. The set $\mathcal{T}$ therefore includes augmentations that are jointly defined on $(X_1, X_2)$ rather than applied independently to each modality. Joint augmentations can remove content from one modality when the other modality retains it, which keeps $I(X'; Y) = I(X; Y)$. The first equality is the condition stated in CoMM, and the second equality states explicitly that the augmentation preserves the information about $Y$. Appendix~\ref{app:assumption} discusses Assumption~\ref{ass:augmentation} further.

Let $X' = t(X)$ for some $t \in \mathcal{T}$, and let $Z = f_\theta(X)$ and $Z' = f_\theta(X')$ denote the representations of the original and augmented inputs. Since the pair $(Z, Z')$ is a deterministic function of $(X, X')$, the data processing inequality yields $I(Z; Z') \leq I(X; X')$. Combined with Assumption~\ref{ass:augmentation}, this leads to the following result:

\begin{lemma}[Multimodal information preservation, \citealp{dufumier2024align}]
\label{lem:multimodal_mi}
Under sufficient expressivity of $f_\theta$, the optimal parameters $\theta^\star$ that maximize $I(Z_\theta; Z'_\theta)$ satisfy $I(Z_{\theta^\star}; Z'_{\theta^\star}) = I(X; X') = I(X; Y)$ and $I(Z_{\theta^\star}; Y) = I(X; Y)$.
\end{lemma}

Since $I(X; Y) = I(X_1, X_2; Y) = R + S + U_1 + U_2$ by Eq.~\eqref{eq:pid}, Lemma~\ref{lem:multimodal_mi} implies that maximizing mutual information between augmented multimodal representations captures all PID components. The two augmented views $X'$ and $X''$ satisfy the same conditions, and the representation at the optimum preserves the task-relevant information of $X$ (Lemma~\ref{lem:two_views}, Appendix~\ref{appendix:proofs}). A corresponding result holds for the unimodal case when a single modality is used as the input. Under Assumption~\ref{ass:augmentation}, the augmented view preserves the information that modality $i$ carries about $Y$, and the unimodal contrastive terms reach this quantity at the optimum.

\begin{lemma}[Unimodal information preservation, \citealp{dufumier2024align}]
\label{lem:unimodal_mi}
Let $\theta^\star$ be the optimal parameters satisfying Lemma~\ref{lem:multimodal_mi}, which under sufficient expressivity of $f_\theta$ also maximize $I(Z_i; Z')$. For the projection $t_i$ that selects modality $i$, the corresponding unimodal representation $Z_i = f_{\theta^\star}(t_i(X))$ satisfies
\begin{equation}
I(Z_i; Z') = I(Z_i; Y) = I(X_i; Y) = R + U_i.
\end{equation}
The equalities follow from $I(X_i; X') = I(X_i; Y)$ under Assumption~\ref{ass:augmentation} and the data processing inequality applied to the chains $Y \to X_i \to Z_i$ and $Z_i \to Y \to Z'$, together with the sufficient expressivity of $f_\theta$ on the unimodal input. The same result holds with $Z''$ in place of $Z'$ (Appendix~\ref{appendix:proofs}).
\end{lemma}

CoMM~\citep{dufumier2024align} introduces Lemmas~\ref{lem:multimodal_mi} and~\ref{lem:unimodal_mi} and uses them to learn task-agnostic representations that capture redundancy, uniqueness, and synergy. Mutual information is estimated using the InfoNCE lower bound~\citep{oord2018representation} with temperature $\tau$, where $\hat{I}_{\mathrm{NCE}}(A, B)$ denotes the InfoNCE estimate of $I(A; B)$, giving the multimodal contrastive objective:

\begin{equation}
    \mathcal{L}_{\text{MCL}} = -\hat{I}_{\text{NCE}}(Z', Z'') - \tfrac{1}{2}\sum_{i=1}^{n}\!\left[\hat{I}_{\text{NCE}}(Z_i, Z') + \hat{I}_{\text{NCE}}(Z_i, Z'')\right].
    \label{eq:comm}
\end{equation}
At the theoretical optimum, the first term captures $R + S + \sum_i U_i$, and the second term captures $\sum_i (R + U_i)$ in the bimodal setting (Lemmas~\ref{lem:multimodal_mi} and~\ref{lem:unimodal_mi}). Although the first term theoretically captures synergy, learning synergistic interactions remains challenging in practice because they require cross-modal computation.

\vspace{-1mm}
\section{Synergy Contrastive Learning}
\vspace{-1mm}

\label{sec:method}

\subsection{The SynCo Framework}
\vspace{-1mm}

\label{sec:framework}

We propose SynCo (Synergy Contrastive Learning), a framework that addresses synergy undertraining by introducing dedicated supervision for synergistic interactions in fused multimodal representations. SynCo introduces a Synergy Head consisting of a linear projector $P$ followed by a residual operation. The projector predicts the fused representation from the unimodal representations, and the component of the fused representation that the projector cannot predict forms the interaction residual. The interaction residual removes the linearly unimodal-predictable component while retaining information not captured by the linear prediction. A contrastive loss is applied to the residual to encourage synergy capture. The overall pipeline is illustrated in Figure~\ref{fig:framework}.

\begin{figure}[t]
    \centering
    \vspace{-2mm}
    \includegraphics[width=0.9\linewidth]{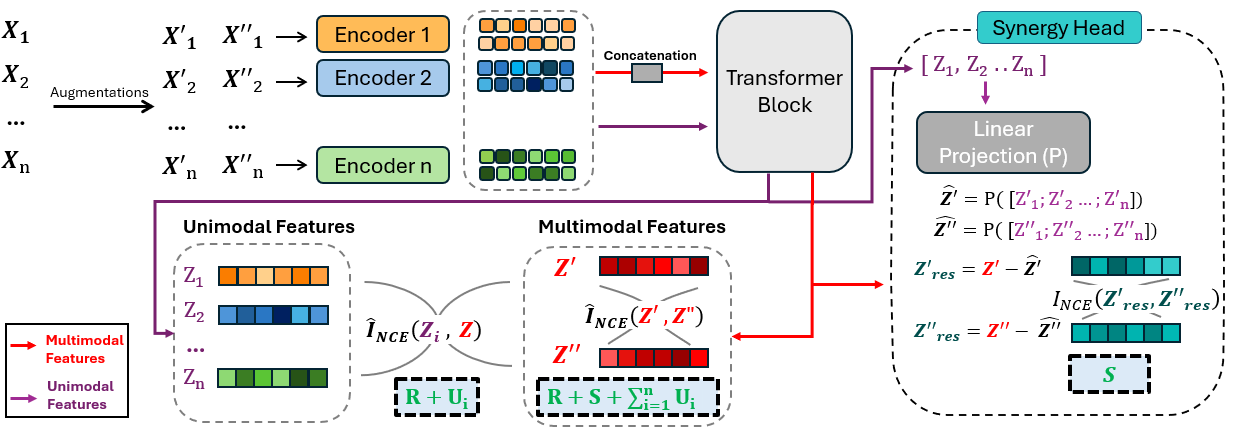}
    \vspace{-1mm}
\caption{Overview of SynCo. Two augmented views are encoded and processed through a multimodal fusion path producing $Z'$ and $Z''$, and a unimodal path producing $Z_1, \ldots, Z_n$. Three contrastive objectives target distinct PID components: $\hat{I}_{\mathrm{NCE}}(Z', Z'')$ captures $R + S + \sum_i U_i$, $\hat{I}_{\mathrm{NCE}}(Z_i, Z')$ and $\hat{I}_{\mathrm{NCE}}(Z_i, Z'')$ capture $R + U_i$, and $\hat{I}_{\mathrm{NCE}}(Z'_{\mathrm{res}}, Z''_{\mathrm{res}})$ provides synergy-focused supervision on the interaction residuals.}
    \label{fig:framework}
    \vspace{-4mm}
\end{figure}

SynCo extends the multimodal contrastive objective $\mathcal{L}_{\text{MCL}}$ of Eq.~\eqref{eq:comm} with two additional terms that operate on the Synergy Head: a contrastive loss $\mathcal{L}_{\text{interaction}}$ on the interaction residuals from the two augmented views, which provides synergy-focused supervision, and a mean-squared-error loss $\mathcal{L}_{\text{pred}}$ that trains the linear projector $P$. The complete SynCo training objective is:

\vspace{-1mm}
\begin{equation}
    \mathcal{L}_{\text{SynCo}} = \mathcal{L}_{\text{MCL}} + \underbrace{\lambda_{\text{int}}\, \mathcal{L}_{\text{interaction}} + \lambda_{\text{pred}}\, \mathcal{L}_{\text{pred}}}_{\text{synergy-targeted terms}},
    \label{eq:synco_total}
\end{equation}

where $\lambda_{\text{int}}$ and $\lambda_{\text{pred}}$ are the weighting coefficients. The theoretical foundation supporting this claim is presented in Section~\ref{sec:theory}.

\subsection{Theoretical Foundation}
\label{sec:theory}


We show that the residual $Z_{\mathrm{res}}^\star$ provides a target for synergy-focused supervision by removing the linearly unimodal-predictable content while retaining additional task-relevant information given the linear prediction. The argument combines an empirically supported hypothesis about the trained encoder with an information-theoretic decomposition.

Let $P: \mathbb{R}^{2D} \to \mathbb{R}^{D}$ be a linear projector trained by mean-squared error to predict the fused multimodal representation $Z$ from the unimodal representations $[Z_1; Z_2]$. We denote the prediction by $\hat{Z} = P([Z_1; Z_2])$ and define the residual $Z_{\mathrm{res}} = Z - \hat{Z}$ as the component of $Z$ not captured by the prediction. We write $P^\star$ for the MSE-optimal projector and $\hat{Z}^\star, Z_{\mathrm{res}}^\star$ for the corresponding prediction and residual.

\begin{hypothesis}[Synergy is not linearly recoverable from unimodal representations]
\label{hyp:synergy}
For the trained encoder $f_\theta$ and the MSE-optimal projector $P^\star$, the linear prediction $\hat{Z}^\star = P^\star([Z_1; Z_2])$ satisfies $I(\hat{Z}^\star; Y) \leq R + U_1 + U_2$.
\end{hypothesis}
The hypothesis bounds the task-relevant information retained by the linear prediction of the fused representation from unimodal representations. We assess it empirically with linear probes in Section~\ref{sec:trifeature}.

To combine this hypothesis with the structure of mutual information, we first establish how task-relevant information is distributed between the prediction $\hat{Z}$ and the residual $Z_{\mathrm{res}}$.

\begin{lemma}[Information decomposition of the residual]
\label{lem:info_decomp}
For any encoder $f_\theta$, any $\hat{Z} = P([Z_1; Z_2])$, and $Z_{\mathrm{res}} = Z - \hat{Z}$,
\begin{equation}
    I(Z_{\mathrm{res}}; Y \mid \hat{Z}) \geq I(Z; Y) - I(\hat{Z}; Y),
    \label{eq:info_decomp}
\end{equation}
with equality when $I(Z; Y) = I(X; Y)$.
\end{lemma}
\begin{proof}
Given $\hat{Z}$, the residual determines $Z$ via $Z = \hat{Z} + Z_{\mathrm{res}}$, so $I(Z_{\mathrm{res}}; Y \mid \hat{Z}) = I(Z; Y \mid \hat{Z})$. The chain rule gives $I(Z, \hat{Z}; Y) = I(\hat{Z}; Y) + I(Z; Y \mid \hat{Z})$, and $I(Z, \hat{Z}; Y) \geq I(Z; Y)$ yields the inequality. Both $Z$ and $\hat{Z}$ are deterministic functions of $X$, so by the data processing inequality $I(Z, \hat{Z}; Y) \leq I(X; Y)$. When $I(Z; Y) = I(X; Y)$, the two bounds give $I(Z, \hat{Z}; Y) = I(Z; Y)$, and the inequality holds with equality.
\end{proof}

\begin{theorem}[Residual concentrates synergistic information]
\label{thm:residual}
Let $Z = f_\theta(X_1, X_2)$ for any encoder $f_\theta$, and let $\Delta = I(X; Y) - I(Z; Y) \geq 0$ denote the task-relevant information that $Z$ does not capture. If Hypothesis~\ref{hyp:synergy} holds for $f_\theta$, then the residual $Z_{\mathrm{res}}^\star = Z - \hat{Z}^\star$ satisfies
\begin{equation}
    I(Z_{\mathrm{res}}^\star; Y \mid \hat{Z}^\star) \geq S - \Delta.
    \label{eq:residual_bound}
\end{equation}
In particular, when $\Delta = 0$, the residual satisfies $I(Z_{\mathrm{res}}^\star; Y \mid \hat{Z}^\star) \geq S$, providing at least $S$ additional task-relevant information given the linear prediction.
\end{theorem}
\begin{proof}
By Lemma~\ref{lem:info_decomp} and Hypothesis~\ref{hyp:synergy}, $I(Z_{\mathrm{res}}^\star; Y \mid \hat{Z}^\star) \geq I(Z; Y) - I(\hat{Z}^\star; Y) \geq I(Z; Y) - (R + U_1 + U_2)$. Substituting $I(Z; Y) = I(X; Y) - \Delta = R + S + U_1 + U_2 - \Delta$ from Eq.~\eqref{eq:pid} gives the lower bound, which reduces to $S$ when $\Delta = 0$.
\end{proof}

\begin{remark}
\label{rem:scope}
Theorem~\ref{thm:residual} holds for any encoder that satisfies Hypothesis~\ref{hyp:synergy}. At the optimum of $\mathcal{L}_{\text{MCL}}$, Lemma~\ref{lem:multimodal_mi} and its extension to two augmented views (Lemma~\ref{lem:two_views}) give $\Delta = 0$. The linear prediction then carries at most $R + U_1 + U_2$ of the task-relevant information, while the residual provides at least $S$ additional task-relevant information given the linear prediction. We refer to this conditional information bound as concentration of synergistic information. The linear projector removes the linearly unimodal-predictable component of the fused representation, while nonlinear $R + U$ content beyond the reach of the projector may remain in $Z_{\mathrm{res}}^\star$.
\end{remark}

\begin{corollary}[Synergy-targeted contrastive learning]
\label{cor:synergy_contrastive}
Let $Z'_{\mathrm{res}}$ and $Z''_{\mathrm{res}}$ be the residuals from two augmented views of the same sample. Under Assumption~\ref{ass:augmentation} and Hypothesis~\ref{hyp:synergy}, the contrastive objective $\hat{I}_{\mathrm{NCE}}(Z'_{\mathrm{res}}, Z''_{\mathrm{res}})$ rewards task-relevant information shared between the interaction residuals, providing a basis for synergy-focused supervision of the multimodal representation.
\end{corollary}
\begin{proof}
Assumption~\ref{ass:augmentation} preserves the total task-relevant information $I(X; Y)$ across both augmented views. Theorem~\ref{thm:residual} applied to each view, with $\Delta = 0$ at the information-preserving optimum (Lemma~\ref{lem:two_views}), shows that each interaction residual provides at least $S$ additional task-relevant information given its linear prediction, which captures the linearly unimodal-predictable component of the fused representation. The InfoNCE objective $\hat{I}_{\mathrm{NCE}}(Z'_{\mathrm{res}}, Z''_{\mathrm{res}})$ is a lower bound on $I(Z'_{\mathrm{res}}; Z''_{\mathrm{res}})$ and rewards information shared between the two residuals. Since the two views are conditionally independent given $Y$ (Lemma~\ref{lem:two_views}), the residuals are also conditionally independent given $Y$. The data processing inequality therefore gives
\begin{equation}
    I(Z'_{\mathrm{res}}; Z''_{\mathrm{res}})
    \leq
    \min\left\{
    I(Z'_{\mathrm{res}}; Y),
    I(Z''_{\mathrm{res}}; Y)
    \right\}.
\end{equation}
Thus, the information shared between the interaction residuals is task-relevant under Assumption~\ref{ass:augmentation}. Applying InfoNCE to these residuals encourages the encoder to preserve shared task-relevant information in the component remaining after subtraction of the linear prediction, providing the motivation for synergy-focused supervision.
\end{proof}


\subsection{Implementation}
\vspace{-1mm}
\label{sec:implementation}

We apply the theoretical setup of Section~\ref{sec:theory} to two augmented views, where $D$ denotes the dimensionality of the fusion output and $n$ the number of modalities. The MSE objective from Section~\ref{sec:theory} is written as
\begin{equation}
    \mathcal{L}_{\text{pred}} = \frac{1}{2}\left[\left\|Z' - P\!\left([Z'_1; Z'_2]\right)\right\|^2 + \left\|Z'' - P\!\left([Z''_1; Z''_2]\right)\right\|^2\right].
    \label{eq:loss_pred}
\end{equation}

The interaction residual is obtained by subtracting $P$'s prediction from the fused representation:
\begin{equation}
    Z'_{\mathrm{res}} = Z' - \hat{Z}', \quad Z''_{\mathrm{res}} = Z'' - \hat{Z}'', \quad \text{where} \quad \hat{Z}' = P\!\left([Z'_1; Z'_2]\right), \; \hat{Z}'' = P\!\left([Z''_1; Z''_2]\right).
    \label{eq:residual_compute}
\end{equation}
The interaction contrastive loss applies InfoNCE to $\ell_2$-normalized residuals from the two augmented views, which form a positive pair targeted at the synergistic component (Corollary~\ref{cor:synergy_contrastive}):
\begin{equation}
    \mathcal{L}_{\text{interaction}} = -\hat{I}_{\text{NCE}}\!\left(\frac{Z'_{\mathrm{res}}}{\|Z'_{\mathrm{res}}\|_2}, \;\frac{Z''_{\mathrm{res}}}{\|Z''_{\mathrm{res}}\|_2}\right),
    \label{eq:loss_interaction}
\end{equation}
using the same temperature $\tau$ as $\mathcal{L}_{\text{MCL}}$. Across all experiments, we use fixed weights $\lambda_{\text{int}} = 0.01$ and $\lambda_{\text{pred}} = 0.1$ without dataset-specific tuning, and linearly warm up $\lambda_{\text{int}}$ over the first 20 epochs to its full value. The sensitivity analyses in Appendices~\ref{app:loss_weights} and~\ref{appendix:warmup_ablation} justify these defaults. Full implementation details are provided in Appendix~\ref{app:implementation_details}.

\paragraph{Gradient isolation.}
Stop-gradient operators separate the parameter updates of $P$ from those of the encoder. In $\mathcal{L}_{\text{pred}}$, we detach both the target $(Z', Z'')$ and the unimodal inputs $(Z'_1, Z'_2, Z''_1, Z''_2)$ to $P$, restricting this loss to update only the parameters of $P$. In $\mathcal{L}_{\text{interaction}}$, we detach the predictions $\hat{Z}'$ and $\hat{Z}''$ before forming the residuals, restricting gradient propagation to $Z'$ and $Z''$ and the multimodal fusion path of the encoder, with no flow into $P$ or the unimodal representations. This construction maintains the convergence of $P$ to the MSE-optimal projector $P^\star$ required by Theorem~\ref{thm:residual} and directs gradient updates to the component of the encoder responsible for cross-modal computation.

\vspace{-1mm}
\section{Experiments}
\label{sec:experiments}
\vspace{-1mm}
We evaluate SynCo on controlled synthetic experiments and large-scale real-world benchmarks. We compare against cross-modal contrastive learning (Cross)~\citep{radford2021learning}, its extension with modality-wise self-supervised terms (Cross+Self)~\citep{yuan2021multimodal}, the disentanglement-based FactorCL~\citep{liang2023factorized}, the multimodal interaction framework CoMM~\citep{dufumier2024align}, and the masking-based InfMasking~\citep{wen2025infmasking} and COrAL~\citep{cissee2026orthogonalized}. We also add the SynCo terms to InfMasking for additional comparison.

On a synthetic bimodal dataset derived from Trifeature~\citep{hermann2020shapes}, we measure how well the learned representations capture each PID component: redundancy, uniqueness, and synergy. We then assess generalization on real-world tasks from MultiBench~\citep{liang2021multibench}, DARai~\citep{kaviani2025hierarchical}, and MM-IMDb~\citep{arevalo2017gated}. These benchmarks span diverse domains and modality combinations. We evaluate all methods via linear probing: after self-supervised pretraining, we freeze the encoder and train a linear classifier or regressor on the learned representations. We report mean and standard deviation over 5 independent runs. Appendix~\ref{app:implementation_details} provides full experimental details.

\subsection{Controlled Experiments on the Trifeature Dataset}
\label{sec:trifeature}
\vspace{-1mm}
We evaluate SynCo on a controlled synthetic benchmark following the experimental protocol of CoMM~\citep{dufumier2024align}. The dataset is derived from Trifeature~\citep{hermann2020shapes}, where each image is defined by three visual attributes: shape, texture, and color, each taking one of ten possible values. Bimodal inputs are constructed by pairing two images. This construction allows precise control over which attributes are shared or unique to each modality and which cross-modal relationships are synergistic.
We construct probing tasks for each PID component. The two constructions below define different pretraining sets, and we pretrain a separate model for each. For redundancy and uniqueness, we pair images by matching shape, making shape the redundant attribute and texture the unique attribute of each modality. We then evaluate linear probes predicting the shared shape (redundancy) and the texture of the first image (uniqueness), both with a chance level of 10\%. For synergy, we define a fixed bijection $\mathcal{M}$ between the ten textures and the ten colors. The self-supervised pretraining set consists of pairs $(X_1, X_2)$ satisfying $\mathcal{M}$, which establishes a cross-modal relationship between $\mathrm{texture}(X_1)$ and $\mathrm{color}(X_2)$ that cannot be recovered from either modality in isolation. The probing task is binary classification of whether a given pair satisfies $\mathcal{M}$, with label $Y = \mathbf{1}[(\mathrm{texture}(X_1), \mathrm{color}(X_2)) \in \mathcal{M}]$. The probe training and test sets are class-balanced, with positive pairs drawn from the pretraining distribution and negative pairs constructed from texture-color combinations outside $\mathcal{M}$, yielding a chance level of 50\%.

\begin{table}[tb]
\vspace{-3pt}
\centering
\scriptsize
\setlength{\abovecaptionskip}{2pt}
\setlength{\belowcaptionskip}{0pt}
\renewcommand{\arraystretch}{0.92}
\begin{minipage}[t]{0.58\linewidth}
\centering
\caption{Linear probing accuracy (\%) for redundancy, uniqueness, and synergy on Trifeature. Bold marks the best value per section. $\dagger$: results from \citet{dufumier2024align}.}
\label{trifeature}
\setlength{\tabcolsep}{3pt}
\begin{tabular}{@{}lccc@{}}
\toprule
\textit{Model} & \textit{redundancy}$\uparrow$ & \textit{uniqueness}$\uparrow$ & \textit{synergy}$\uparrow$ \\
\midrule
Cross$^\dagger$ & \textbf{100.0} & 11.6 & 50.0 \\
Cross+Self$^\dagger$ & 99.7 & 86.9 & 50.0 \\
FactorCL$^\dagger$ & 99.8 & 62.5 & 46.5 \\
COrAL & 99.7$_{\pm0.06}$ & \textbf{91.4}$_{\pm0.62}$ & 73.91$_{\pm0.51}$ \\
CoMM & 99.9$_{\pm0.05}$ & 86.6$_{\pm1.32}$ & 71.6$_{\pm1.07}$ \\
SynCo (ours) & 99.9$_{\pm0.21}$ & 90.15$_{\pm1.14}$ & \textbf{77.58}$_{\pm0.74}$ \\
\midrule
InfMasking & \textbf{99.9}$_{\pm0.07}$ & 89.1$_{\pm0.95}$ & 76.12$_{\pm1.42}$ \\
SynCo (ours) + InfMasking & \textbf{99.9}$_{\pm0.24}$ & \textbf{90.51}$_{\pm1.27}$ & \textbf{77.89}$_{\pm0.85}$ \\
\bottomrule
\end{tabular}
\vspace{-3pt}
\end{minipage}\hfill
\begin{minipage}[t]{0.39\linewidth}
\centering
\caption{Linear probing accuracy (\%) on the unimodal representations and on the three components of SynCo's representation decomposition across PID tasks on the Trifeature benchmark. $[Z_1; Z_2]$: unimodal; $Z$: multimodal; $\hat{Z}$: linear prediction; $Z_{\mathrm{res}}$: residual.}
\label{tab:decomp}
\setlength{\tabcolsep}{3pt}
\begin{tabular}{@{}lccc@{}}
\toprule
\textit{Representation} & \textit{Redun.}$\uparrow$ & \textit{Uniq.}$\uparrow$ & \textit{Syn.}$\uparrow$ \\
\midrule
$[Z_1; Z_2]$       & 99.9$_{\pm0.16}$ & 89.42$_{\pm2.33}$ & 50.00$_{\pm0.00}$ \\
$Z$                & 99.9$_{\pm0.21}$ & 90.15$_{\pm1.14}$ & 77.58$_{\pm0.74}$ \\
$\hat{Z}$          & 99.9$_{\pm0.18}$ & 89.26$_{\pm2.41}$ & 50.00$_{\pm0.00}$ \\
$Z_{\mathrm{res}}$ & 81.1$_{\pm2.17}$ & 57.40$_{\pm2.25}$ & \textbf{87.24}$_{\pm1.78}$ \\
\bottomrule
\end{tabular}
\vspace{-4pt}
\end{minipage}
\vspace{-7pt}
\end{table}

Results are reported in Table~\ref{trifeature}. Cross-modal contrastive training alone (Cross) achieves perfect redundancy but fails entirely on both uniqueness and synergy. Adding per-modality self-supervised terms (Cross+Self) recovers uniqueness but fails to improve synergy beyond chance. FactorCL performs below chance on synergy, confirming that disentanglement-based objectives do not resolve the synergy deficit. CoMM is the first method to capture all three PID components, and InfMasking further improves synergy through stochastic feature masking. SynCo surpasses CoMM by $+3.55\%$ on uniqueness and $+5.98\%$ on synergy while maintaining near-perfect redundancy, and outperforms InfMasking by $+1.46\%$ on synergy despite using a substantially simpler mechanism. COrAL achieves the highest uniqueness, while SynCo exceeds it by $+3.67\%$ on synergy. Combining SynCo with InfMasking yields the highest synergy among all methods. The consistent gains over CoMM and InfMasking confirm that our method provides effective synergy-focused supervision of the learned representations.

\paragraph{Linear probing on decomposed representations.}
We assess Hypothesis~\ref{hyp:synergy} empirically by training linear probes on the concatenated unimodal representations $[Z_1; Z_2]$ and on the three components of SynCo's representation: the multimodal representation $Z$, the linear prediction $\hat{Z} = P([Z_1; Z_2])$, and the residual $Z_{\mathrm{res}} = Z - \hat{Z}$. Table~\ref{tab:decomp} reports the results. Neither the concatenated unimodal representations nor the linear prediction $\hat{Z}$ exceeds chance-level accuracy on the synergy task, indicating that the synergy label is not linearly decodable from the unimodal representations or from their linear prediction, consistent with Hypothesis~\ref{hyp:synergy}. The residual $Z_{\mathrm{res}}$ surpasses the multimodal representation $Z$ on synergy by $+9.66\%$, consistent with Theorem~\ref{thm:residual}. The residual also contains nonzero $R$ and $U$ content, consistent with Remark~\ref{rem:scope}, since the linear projector removes only the linearly unimodal-predictable component and nonlinear $R + U$ content may remain in the residual.

\subsection{Experiments on Real-world Datasets}
\label{sec:experiments_realworld}
We evaluate SynCo on real-world multimodal datasets from MultiBench~\citep{liang2021multibench}, DARai~\citep{kaviani2025hierarchical}, and MM-IMDb~\citep{arevalo2017gated}. These benchmarks span diverse domains, modality types, and task formulations. Following prior work~\citep{liang2023factorized, dufumier2024align, wen2025infmasking}, we use the same data preprocessing steps, modality configurations, and backbone networks across all methods for MultiBench and MM-IMDb. For DARai, we adopt the sensor encoder from~\citep{khaertdinov2021contrastive}. Implementation details are provided in Appendix~\ref{app:implementation_details}.

\begin{table*}[b]
  \vspace{-2mm}
  \centering
  \scriptsize
  \setlength{\tabcolsep}{2pt}
  \renewcommand{\arraystretch}{0.8}
  \setlength{\aboverulesep}{0.2ex}
  \setlength{\belowrulesep}{0.2ex}
  \setlength{\abovecaptionskip}{2pt}
  \setlength{\belowcaptionskip}{0pt}
\caption{Linear probing on MultiBench: MSE ($\times 10^{-4}$) for regression and top-1 accuracy (\%) for classification. Bold: best within each section. $\dagger$: results from \cite{dufumier2024align}. $^*$: classification average.}
  \begin{tabular}{@{}lcccccc@{}}
    \toprule
    \multirow{2}{*}{\textit{Model}} & \textit{Regr.} & \multicolumn{5}{c}{\textit{Classification}} \\[-1pt]
    \cmidrule(lr){2-2} \cmidrule(lr){3-7}
    & \textit{V\&T}$\downarrow$ & \textit{MIMIC}$\uparrow$ & \textit{MOSI}$\uparrow$ & \textit{UR-FUNNY}$\uparrow$ & \textit{MUSTARD}$\uparrow$ & \textbf{Avg.}$^*$ $\uparrow$ \\
    \midrule
    Cross$^\dagger$ & $33.09_{\pm3.67}$ & $66.70_{\pm0.10}$ & $47.80_{\pm1.80}$ & $50.10_{\pm1.90}$ & $53.50_{\pm2.90}$ & $54.53$ \\
    Cross+Self$^\dagger$ & $7.56_{\pm0.31}$ & $65.49_{\pm0.00}$ & $49.00_{\pm1.10}$ & $59.90_{\pm0.90}$ & $53.90_{\pm4.00}$ & $57.07$ \\
    FactorCL$^\dagger$ & $10.82_{\pm0.56}$ & $67.30_{\pm0.00}$ & $51.20_{\pm1.60}$ & $60.50_{\pm0.80}$ & $55.80_{\pm0.90}$ & $58.70$ \\
    COrAL & $5.40_{\pm0.28}$ & $\mathbf{67.95}_{\pm0.21}$ & $66.40_{\pm0.84}$ & $64.80_{\pm0.47}$ & $66.10_{\pm1.02}$ & $66.31$ \\
    CoMM & $6.85_{\pm2.13}$ & $65.92_{\pm0.35}$ & $63.70_{\pm2.50}$ & $63.25_{\pm0.81}$ & $64.58_{\pm1.50}$ & $64.36$ \\
    SynCo (ours) & $\mathbf{5.21}_{\pm0.34}$ & $67.89_{\pm0.56}$ & $\mathbf{67.92}_{\pm1.06}$ & $\mathbf{66.52}_{\pm0.53}$ & $\mathbf{67.02}_{\pm1.12}$ & $\mathbf{67.34}$ \\
    \midrule
    InfMasking & $\mathbf{4.95}_{\pm0.56}$ & $67.53_{\pm0.32}$ & $67.12_{\pm0.75}$ & $64.11_{\pm0.95}$ & $65.62_{\pm0.90}$ & $66.10$ \\
    SynCo + InfMasking & $5.12_{\pm0.21}$ & $\mathbf{69.52}_{\pm0.49}$ & $\mathbf{67.81}_{\pm1.18}$ & $\mathbf{66.91}_{\pm0.62}$ & $\mathbf{67.21}_{\pm1.65}$ & $\mathbf{67.86}$ \\
    \bottomrule
  \end{tabular}
  \label{tab:multibench}
  \vspace{-4mm}
\end{table*}

\paragraph{MultiBench.}
We use five datasets from MultiBench: Vision\&Touch~\citep{lee2020making} for end-effector position regression from visual and proprioceptive modalities, MIMIC~\citep{johnson2016mimic} for binary classification of respiratory-system diagnoses from tabular and time-series modalities, and MOSI~\citep{zadeh2016multimodal}, UR-FUNNY~\citep{hasan2019ur}, and MUSTARD~\citep{castro2019towards} for sentiment analysis, humor detection, and sarcasm detection from visual and textual modalities. Appendix~\ref{sec:datasets} describes each dataset in detail. Table~\ref{tab:multibench} reports the results. SynCo outperforms CoMM, InfMasking, and COrAL on average classification accuracy and achieves the highest accuracy on MOSI among standalone methods. Combining SynCo with InfMasking yields the best overall performance, with the highest accuracy on MIMIC, UR-FUNNY, and MUSTARD, and the strongest average across classification tasks. On the regression task (V\&T), SynCo remains competitive with InfMasking, outperforms COrAL, and substantially outperforms all other baselines.

We report results on the wearable-sensor benchmark DARai~\citep{kaviani2025hierarchical} in Appendix~\ref{app:darai} and results in the trimodal setting in Appendix~\ref{app:trimodal}.


\paragraph{MM-IMDb.}
\begin{wraptable}{r}{0.44\textwidth}
\vspace{-8pt}
\centering
\scriptsize
\setlength{\tabcolsep}{3pt}
\renewcommand{\arraystretch}{0.85}
\setlength{\aboverulesep}{0.2ex}
\setlength{\belowrulesep}{0.2ex}
\setlength{\abovecaptionskip}{2pt}
\setlength{\belowcaptionskip}{2pt}
\caption{Linear probing F1-scores (weighted and macro, in \%) on MM-IMDb (vision and language). All self-supervised methods use a CLIP backbone. $\dagger$ denotes results from \cite{dufumier2024align}.}
\label{tab:imdb}
\begin{tabular}{@{}lcc@{}}
\toprule
\textit{Model} & \textit{w-F1}$\uparrow$ & \textit{m-F1}$\uparrow$ \\
\midrule
CLIP$^{\dagger}$~\cite{radford2021learning} & $54.49_{\pm0.19}$ & $44.94_{\pm0.30}$ \\
SLIP$^{\dagger}$~\cite{mu2022slip} & $56.54_{\pm0.19}$ & $47.35_{\pm0.27}$ \\
CoMM~\cite{dufumier2024align} & $61.29_{\pm0.73}$ & $53.79_{\pm0.22}$ \\
SynCo (ours) & $\mathbf{62.12}_{\pm0.17}$ & $\mathbf{55.41}_{\pm0.16}$ \\
\midrule
InfMasking~\cite{wen2025infmasking} & $61.97_{\pm0.28}$ & $55.28_{\pm0.15}$ \\
SynCo + InfMask. & $\mathbf{62.23}_{\pm0.27}$ & $\mathbf{55.64}_{\pm0.22}$ \\
\bottomrule
\end{tabular}
\vspace{-4pt}
\end{wraptable}
MM-IMDb~\citep{arevalo2017gated} is a multi-label movie genre classification dataset with visual (poster) and textual (plot summary) modalities. The dataset is challenging due to substantial class imbalance across genres and a large semantic gap between the two modalities. We compare against CLIP~\citep{radford2021learning}, SLIP~\citep{mu2022slip}, CoMM~\citep{dufumier2024align}, and InfMasking~\citep{wen2025infmasking}, with all methods using a CLIP backbone. Table~\ref{tab:imdb} reports the results. SynCo outperforms both CoMM and standalone InfMasking on weighted-F1 and macro-F1 and achieves the best standalone performance. Combining SynCo with InfMasking yields a further gain and the best overall result on both metrics.

\subsection{Ablation Studies}
\label{sec:ablation}

\begin{wraptable}{r}{0.33\textwidth}
\vspace{-12pt}
\centering
\scriptsize
\setlength{\tabcolsep}{3pt}
\renewcommand{\arraystretch}{0.9}
\caption{Synergy probe accuracy (\%) on Trifeature across projector architectures. Chance level is 50\%.}
\label{tab:projector_ablation}
\begin{tabular}{@{}lccc@{}}
\toprule
\textit{Projector} & $Z$\,$\uparrow$ & $\hat{Z}$\,$\downarrow$ & $Z_{\mathrm{res}}$\,$\uparrow$ \\
\midrule
Linear (ours)  & \textbf{77.58} & \textbf{50.00} & \textbf{87.24} \\
MLP-2          & 74.14 & 59.75 & 75.31 \\
MLP-3          & 73.11 & 62.10 & 74.13 \\
MLP-5          & 71.25 & 64.81 & 72.84 \\
Gated Residual & 70.56 & 68.19 & 69.41 \\
Transformer-4  & 66.45 & 71.47 & 62.04 \\
\bottomrule
\end{tabular}
\vspace{-10pt}
\end{wraptable}
\textbf{Ablation on Projector Architecture.}
The residual $Z_{\mathrm{res}} = Z - \hat{Z}$ contains the part of $Z$ that $P$ cannot predict, so a higher-capacity projector may capture more of the fused representation, leaving less information in the residual. To assess this design choice, we replace $P$ with five higher-capacity variants: 2-layer, 3-layer, and 5-layer MLPs, a Gated Residual projector, and a 4-layer Transformer. Holding all other components fixed, we evaluate linear probing accuracy on the Trifeature synergy task for $Z$, $\hat{Z}$, and $Z_{\mathrm{res}}$. Architectural details are provided in Appendix~\ref{appendix:projector_variants}.

Table~\ref{tab:projector_ablation} shows that the prediction $\hat{Z}$ of the linear projector remains at chance on the synergy task, while higher-capacity projectors increasingly capture information useful for the synergy task in $\hat{Z}$. As projector capacity increases, synergy probe accuracy decreases in both $Z_{\mathrm{res}}$ and $Z$, weakening the effectiveness of $\mathcal{L}_{\mathrm{interaction}}$ for synergy-focused supervision. The linear projector achieves the highest synergy probe accuracy in $Z_{\mathrm{res}}$ while removing the linearly unimodal-predictable component of the fused representation. Additional ablations on the loss terms, loss weights, and warmup length are provided in Appendices~\ref{app:loss_ablation}, \ref{app:loss_weights}, and~\ref{appendix:warmup_ablation}, respectively.

\vspace{-1mm}
\section{Conclusion}

\label{sec:conclusion}
We presented SynCo, a contrastive multimodal framework that addresses synergy undertraining by supervising an interaction residual that removes the linearly unimodal-predictable component of the fused representation. The Synergy Head provides synergy-focused supervision at negligible computational cost. On the controlled Trifeature benchmark, SynCo achieves state-of-the-art synergy capture and improves uniqueness while maintaining near-perfect redundancy. On real-world benchmarks from MultiBench, DARai, and MM-IMDb, SynCo consistently outperforms or matches prior methods across diverse modality combinations and task types, including vision-language, sensor-based, and clinical domains. SynCo adds only two loss terms and a lightweight prediction head without modifying the underlying fusion architecture, making it a straightforward plug-in for existing multimodal contrastive frameworks.

\newpage

\newpage
\pagebreak
\bibliography{iclr2027_conference}
\bibliographystyle{iclr2027_conference}
\newpage
\pagebreak


\appendix
\section{Appendix Implementation Details}\label{app:implementation_details}

This section describes the encoder architectures, modality-specific augmentations, latent converters, and optimization settings used across all experiments.

\subsection{Encoder Architectures}

\paragraph{Trifeature.} Both modalities are visual, so we use a shared AlexNet backbone~\citep{krizhevsky2012imagenet} with a $512$-dimensional embedding space, matching the encoder choice of CoMM. We remove the final average-pooling layer and apply a linear patch-embedding module~\citep{dosovitskiy2020image} to the resulting $6 \times 6$ feature maps. Fixed two-dimensional sine-cosine positional embeddings are added before the fusion transformer.

\paragraph{MIMIC.} The two modalities of MIMIC are tabular and time-series. The tabular branch uses a two-layer MLP with a $10$-dimensional hidden size and a $10$-dimensional output, and the time-series branch uses a GRU with a $512$-dimensional hidden state. A feature tokenizer~\citep{gorishniy2021revisiting} maps the tabular output to a $512$-dimensional embedding space, and no latent converter is applied to the time-series branch.

\paragraph{MOSI, UR-FUNNY, and MUSTARD.} For these three datasets the modalities are pre-extracted visual and textual features. We follow FactorCL~\citep{liang2023factorized} and CoMM in using a five-head, five-layer Transformer with a $40$-dimensional embedding space for each modality, with no latent converter.

\paragraph{Vision\&Touch.} For the end-effector regression task we use the visual and force-torque modalities. Images are encoded with a ResNet-18~\citep{he2016deep} producing $128$-channel spatial feature maps, and force-torque signals are encoded with a five-layer causal convolutional network producing a $128$-dimensional output, following the original Vision\&Touch protocol~\citep{lee2020making}. A patch-embedding module is applied to the ResNet-18 feature maps, and the force-torque embeddings are tokenized through a feature tokenizer.

\paragraph{DARai.} For each wearable modality we adopt the encoder architecture introduced for sensor-based human activity recognition by Khaertdinov et al.~\citep{khaertdinov2021contrastive}, which combines a stack of three one-dimensional convolutional layers, each followed by batch normalization and a ReLU activation and using reflective padding to preserve the temporal length, with a transformer self-attention block. We fix the kernel size to three throughout the convolutional stack and use eight attention heads in the transformer block. All sensor streams are resampled to a common length of $100$ time steps before encoding, so that the IMU, EMG, and insole branches operate on tensors of identical temporal extent. The three branches share this backbone configuration and differ only in the channel count of the input layer, which is determined by the corresponding sensor dimensionality.

\paragraph{MM-IMDb.} The visual encoder is a ViT-B/$32$~\citep{dosovitskiy2020image} pre-trained with CLIP~\citep{radford2021learning}, and the textual encoder is the multilingual Sentence-BERT model distilled from CLIP~\citep{reimers2019sentence}. Both encoders are frozen during pretraining, and SynCo operates directly on the token embeddings produced by these backbones, without latent converters.

\subsection{Fusion Transformer and Linear Projector}

The fusion module is a Transformer encoder layer with multi-head self-attention, a feed-forward block, residual connections, and layer normalization. We use a single layer with eight heads in the bimodal setting and two layers with eight heads in the trimodal setting. A learnable \texttt{[CLS]} token is prepended to the input sequence, and the corresponding output token is read out as the fused representation $Z$. The unimodal representation $Z_i$ is the \texttt{[CLS]} output of the same modality encoders and fusion transformer when all modalities except the $i$-th are masked, so the fused and unimodal representations lie in the same space $\mathbb{R}^{D}$. The linear projector $P$ in SynCo is implemented as a single linear layer mapping $\mathbb{R}^{nD}$ to $\mathbb{R}^{D}$, where $D$ is the dimensionality of $Z$ and $n$ is the number of modalities. It is applied to the concatenation of the unimodal representations $[Z_1; \ldots; Z_n]$ and produces $\hat{Z}$.

\subsection{Residual Computation and Training Dynamics}

\paragraph{Gradient structure.}
The three stop-gradient placements introduced in Section~\ref{sec:implementation} yield a clear separation of learning roles. The linear projector $P$ receives gradients only from $\mathcal{L}_{\text{pred}}$ and learns to approximate $Z$ from $[Z_1; Z_2]$ to the extent permitted by a linear function. The encoder and fusion transformer receive gradients from two distinct sources: the multimodal contrastive loss $\mathcal{L}_{\text{MCL}}$, which propagates through the projection head and operates in the projected space, and the interaction contrastive loss $\mathcal{L}_{\text{interaction}}$, which propagates through $Z'_{\mathrm{res}}$ and $Z''_{\mathrm{res}}$ and operates in the raw fusion space. The interaction contrastive loss specifically supervises features that $P$ cannot predict, which by Theorem~\ref{thm:residual} provide at least $S$ additional task-relevant information given the linear prediction. The construction does not introduce conflicting optimization objectives. $P$ is trained to minimize prediction error and therefore cannot increase the residual norm. The encoder receives no gradient that would degrade the unimodal representations to enlarge the residual, since the predictions $\hat{Z}'$ and $\hat{Z}''$ are detached in $\mathcal{L}_{\text{interaction}}$ and no gradient from this loss passes through $P$ or the unimodal representations. The modality encoders and fusion transformer are shared between the fused and unimodal paths, so updates driven by $\mathcal{L}_{\text{interaction}}$ still change the unimodal representations indirectly, and the unimodal terms of $\mathcal{L}_{\text{MCL}}$ keep them informative about $R + U_i$ (Lemma~\ref{lem:unimodal_mi}).

\paragraph{Training dynamics.}
SynCo exhibits a natural curriculum arising from the interplay of prediction and interaction objectives, and the implementation additionally performs an explicit linear warmup of the interaction weight. Early in training the projector $P$ is randomly initialized and the residual approximates the fused representation, so the interaction contrastive term initially provides a broad discriminative signal. To avoid the interaction term dominating too early, we linearly scale the effective interaction weight from zero to its configured value over a fixed number of warmup epochs:
\[
\lambda_{\text{int}}^{\text{eff}} = \lambda_{\text{int}} \cdot \min\!\left(1, \frac{\text{epoch}}{\text{warmup\_epochs}}\right).
\]
In our runs, \texttt{warmup\_epochs} is set to 20, so $\lambda_{\text{int}}$ is fully active only after the warmup period. The prediction loss $\mathcal{L}_{\text{pred}}$ is applied from the start, while the interaction loss is gradually annealed in as $P$ converges, so that the residual progressively excludes the linearly unimodal-predictable component. Beyond this linear warmup of $\lambda_{\text{int}}$, no additional scheduling or curriculum is required.

\subsection{Modality-Specific Augmentations}

For raw images on Trifeature, MM-IMDb, and Vision\&Touch, we adopt the SimCLR augmentation pipeline~\citep{chen2020simple}. Tabular features in MIMIC are perturbed with additive Gaussian noise applied independently to each feature dimension. For pre-extracted time-series features in MIMIC, MOSI, UR-FUNNY, MUSTARD, and the force-torque modality of Vision\&Touch, we apply Gaussian noise together with random dropping of a uniformly sampled fraction of the sequence between 0\% and 80\%; dropped timesteps are zeroed rather than removed. For the wearable signals in DARai, we do not use sequence dropping; instead we apply sensor-level augmentations including additive Gaussian jittering, optional scaling, random sign and rotation, permutation of segments, and channel shuffling, as implemented in the DARai transforms. For raw text on MM-IMDb, we apply random token masking with probability $15\%$, where tokens are selected uniformly with pad, CLS, and SEP tokens excluded, by replacing selected token ids with the tokenizer's mask token. This is a direct mask-replacement scheme applied only during training and does not implement the BERT $80/10/10$ replace/keep schedule.

\subsection{Optimization}

All experiments use the AdamW optimizer~\citep{loshchilov2017decoupled}, except for DARai where we use Adam following the protocol of the source benchmark. The learning rate is set to $3 \times 10^{-4}$ with weight decay $10^{-4}$ on Trifeature, $1 \times 10^{-3}$ with weight decay $10^{-2}$ on MIMIC, MOSI, UR-FUNNY, and MUSTARD, $1 \times 10^{-4}$ with weight decay $10^{-2}$ on Vision\&Touch and MM-IMDb, and $1 \times 10^{-3}$ on DARai. Pretraining uses a batch size of $64$ across all datasets. For MM-IMDb we additionally use a cosine schedule with a final value of $10^{-6}$ and a ten-epoch warmup. Pretraining is run for $100$ epochs on every dataset except MM-IMDb, where we run $70$ epochs following CoMM, and DARai, where we pretrain for $300$ epochs following the source benchmark. The InfoNCE critic is implemented as a three-layer MLP with $512$-dimensional hidden layers and a $256$-dimensional output, matching the projection head architecture of SimCLR. The temperature $\tau$ in every InfoNCE term is set to $0.1$ and shared across $\mathcal{L}_{\text{MCL}}$ and $\mathcal{L}_{\text{interaction}}$. Across all experiments we use the fixed coefficients $\lambda_{\text{int}} = 0.01$ and $\lambda_{\text{pred}} = 0.1$ for the two SynCo-specific terms. The sensitivity analysis in Appendix~\ref{app:loss_weights} justifies these defaults. All experiments are performed on a single NVIDIA A40 GPU with 48 GB of memory.

\subsection{Linear Probing Protocol}

After pretraining, the encoder and fusion transformer are frozen, and a single linear layer is trained on top of the multimodal representation $Z$ for $100$ epochs. Classification tasks are optimized with cross-entropy loss, and the Vision\&Touch end-effector task is optimized with mean-squared error. We use early stopping on the validation set and report mean and standard deviation across five independent runs.

\subsection{Projector Architecture Ablation Details}
\label{appendix:projector_variants}
This appendix details the projector variants evaluated in the ablation reported in Section~\ref{sec:ablation}. All variants implement the mapping $P : \mathbb{R}^{nD} \to \mathbb{R}^{D}$ from the concatenated unimodal representations $[Z_1; \ldots; Z_n]$ to the fused representation space, where $D$ is the dimensionality of $Z$ and $n$ is the number of modalities. Throughout the ablation, only the architecture of $P$ is varied; the encoder, fusion transformer, augmentations, optimizer, loss coefficients, and training schedule are held fixed at the values reported in Appendix~\ref{app:implementation_details}, and each variant is trained jointly with the SynCo objective. The stop-gradient configuration from Section~\ref{sec:method} continues to apply, so $\mathcal{L}_{\text{pred}}$ updates only the parameters of $P$ regardless of its architecture. All variants are described under the bimodal Trifeature setting with $n = 2$ and $D = 512$.

\paragraph{Linear (baseline).} A single fully connected layer with weight matrix $W \in \mathbb{R}^{D \times nD}$ and bias $b \in \mathbb{R}^{D}$, giving $\hat{Z} = W[Z_1; Z_2] + b$. This is the projector used in all main-text results.

\paragraph{2-layer MLP.} A two-layer feed-forward network. The first layer maps $\mathbb{R}^{nD}$ to a hidden dimension $h = (nD + D)/2$, followed by a ReLU activation. The second layer maps the hidden representation to $\mathbb{R}^{D}$.

\paragraph{3-layer and 5-layer MLPs.} Stacks of three and five fully connected layers respectively, with ReLU activations between consecutive layers and constant intermediate widths equal to the output dimension $D$. The first layer maps $\mathbb{R}^{nD}$ to $\mathbb{R}^{D}$ and all subsequent layers operate within $\mathbb{R}^{D}$.

\paragraph{Gated Residual.} A linear branch combined additively with a parallel nonlinear branch:
\begin{equation}
\hat{Z} = W_{\ell}[Z_1; Z_2] + b_{\ell} + \gamma \cdot \mathrm{MLP}([Z_1; Z_2]),
\end{equation}
where $W_{\ell} \in \mathbb{R}^{D \times nD}$ and $b_{\ell} \in \mathbb{R}^{D}$ parameterize the linear branch, $\mathrm{MLP}$ is a two-layer feed-forward network with hidden dimension $D$ and ReLU activation, and $\gamma \in \mathbb{R}$ is a learnable scalar gate initialized to zero. Initializing $\gamma$ to zero renders the projector strictly linear at the start of training, so the residual loss begins from the same starting point as the linear baseline, and $\gamma$ grows only when the nonlinear branch reduces $\mathcal{L}_{\text{pred}}$.

\paragraph{Transformer (4-layer).} The concatenated input $[Z_1; Z_2] \in \mathbb{R}^{nD}$ is reshaped into $n$ modality tokens of dimension $D$, and learned modality-specific embeddings $E_1, \dots, E_n \in \mathbb{R}^{D}$ are added to the corresponding tokens. The resulting sequence is processed by a stack of four Transformer encoder layers, each containing multi-head self-attention with eight heads, a position-wise feed-forward sublayer with hidden dimension $4D$ and GELU activation, and pre-layer normalization on both sublayers. The output tokens are flattened to $\mathbb{R}^{nD}$ and passed through a two-layer MLP head with hidden dimension $D$ and GELU activation, producing $\hat{Z} \in \mathbb{R}^{D}$. This variant comprises approximately 14.7M parameters.

\subsection{Combining SynCo with InfMasking}
\label{app:synco_infmasking}

Combining SynCo with InfMasking augments the multimodal contrastive terms with a masking objective
$\mathcal{L}_{\text{mask}}$ that stochastically replaces fusion-input tokens with a
learned mask token and maximizes agreement between the masked and unmasked
multimodal representations. Its total loss is the uniform mean of the masking
terms and the contrastive terms of $\mathcal{L}_{\text{MCL}}$; we denote it
$\mathcal{L}_{\text{InfMasking}}$ and leave it untouched. The combined objective is
\begin{equation}
  \mathcal{L}_{\text{SynCo+InfMasking}} = \mathcal{L}_{\text{InfMasking}}
  + \lambda_{\text{int}}\,\mathcal{L}_{\text{interaction}}
  + \lambda_{\text{pred}}\,\mathcal{L}_{\text{pred}}.
\end{equation}

The Synergy Head operates on the unmasked representations. The residuals $Z'_{\mathrm{res}}$ and $Z''_{\mathrm{res}}$ are formed from the fused representations $Z'$ and $Z''$ computed without InfMasking's feature masking, and the projector $P$ receives the unimodal representations $Z'_i$ and $Z''_i$ defined in Section~\ref{sec:preliminaries}. The stochastically masked representations enter only $\mathcal{L}_{\text{mask}}$. InfMasking perturbs the input to fusion, while SynCo decomposes the unmasked output of fusion, and the two mechanisms act on different quantities.

The stop-gradient configuration of Section~\ref{sec:implementation} is unchanged, and $\mathcal{L}_{\text{mask}}$ does not update $P$. We use the InfMasking hyperparameters of~\citet{wen2025infmasking} for the masking ratio, the number of masked views, and $\lambda_{\text{mask}}$, together with the fixed SynCo coefficients $\lambda_{\text{int}} = 0.01$ and $\lambda_{\text{pred}} = 0.1$ and the 20-epoch warmup of $\lambda_{\text{int}}$. All other settings follow Appendix~\ref{app:implementation_details}. The combination requires no forward passes beyond those of InfMasking, and its only additional cost is the linear projector.

\section{Appendix Additional Experiments}

\subsection{Experiments on DARai}
\label{app:darai}

DARai~\citep{kaviani2025hierarchical} is a human activity recognition dataset with multiple wearable sensor modalities, and we use it to evaluate whether SynCo's gains generalize across modality pairings. We evaluate three bimodal configurations, IMU + Insole, IMU + EMG, and EMG + Insole, and compare FactorCL, CoMM, InfMasking, and SynCo under identical preprocessing and backbones. Appendix~\ref{sec:datasets} describes the sensor modalities and the three configurations in detail.

\begin{table}[ht]
  \centering
  \footnotesize
  \setlength{\tabcolsep}{6pt}
  \renewcommand{\arraystretch}{0.9}
  \caption{Linear probing top-1 accuracy (\%) on DARai across three bimodal sensor configurations. Bold values indicate the best within each section.}
  \label{tab:darai}
  \begin{tabular}{@{}lcccc@{}}
    \toprule
    \textit{Model} & \textit{IMU+Insole}$\uparrow$ & \textit{IMU+EMG}$\uparrow$ & \textit{EMG+Insole}$\uparrow$ & \textbf{Avg.}$\uparrow$ \\
    \midrule
    FactorCL & $36.74_{\pm1.15}$ & $34.18_{\pm1.05}$ & $30.92_{\pm1.18}$ & $33.95$ \\
    CoMM & $42.12_{\pm0.81}$ & $41.85_{\pm0.94}$ & $32.21_{\pm1.05}$ & $38.73$ \\
    SynCo (ours) & $\mathbf{45.52}_{\pm0.54}$ & $\mathbf{43.78}_{\pm0.42}$ & $\mathbf{36.15}_{\pm0.59}$ & $\mathbf{41.82}$ \\
    \midrule
    InfMasking & $43.85_{\pm1.13}$ & $42.32_{\pm0.84}$ & $36.04_{\pm1.12}$ & $40.74$ \\
    SynCo (ours) + InfMasking & $\mathbf{45.21}_{\pm0.98}$ & $\mathbf{44.15}_{\pm1.07}$ & $\mathbf{36.38}_{\pm1.06}$ & $\mathbf{41.91}$ \\
    \bottomrule
  \end{tabular}
\end{table}

Table~\ref{tab:darai} reports the results. SynCo outperforms CoMM on every configuration and achieves the best standalone average. Combining SynCo with InfMasking yields an additional gain on average, with the best mean accuracy on IMU + EMG and EMG + Insole, while standalone SynCo leads on IMU + Insole. The consistent gains across configurations show that SynCo generalizes beyond vision-language settings to wearable sensor data.

\subsection{Experiments with 3 Modalities}
\label{app:trimodal}
We extend our evaluation to the trimodal setting using MOSI, UR-FUNNY, and MUSTARD with their full text, audio, and video configurations. While Section~\ref{sec:experiments_realworld} reports results on the bimodal variants of these datasets, the trimodal setting introduces additional cross-modal interactions that test whether the synergy-targeted supervision in SynCo scales beyond two modalities. We follow the same experimental protocol as the bimodal evaluation, comparing against CoMM and InfMasking under identical preprocessing, backbone networks, and linear probing procedures.
\begin{table}[ht]
  \centering
  \footnotesize
  \setlength{\tabcolsep}{6pt}
  \renewcommand{\arraystretch}{0.9}
  \caption{Linear probing top-1 accuracy (in \%) on trimodal classification benchmarks (text, audio, video).}
  \begin{tabular}{@{}lcccc@{}}
    \toprule
    \textit{Model} & \textit{MOSI}$\uparrow$ & \textit{UR-FUNNY}$\uparrow$ & \textit{MUSTARD}$\uparrow$ & \textbf{Average}$\uparrow$ \\
    \midrule
    CoMM \cite{dufumier2024align} & $64.00_{\pm 2.33}$ & $64.74_{\pm 0.71}$ & $65.52_{\pm 1.42}$ & $64.75$ \\
    SynCo (ours) & $\mathbf{68.11}_{\pm 0.92}$ & $\mathbf{66.85}_{\pm 0.55}$ & $\mathbf{68.12}_{\pm 1.24}$ & $\mathbf{67.69}$ \\
    \midrule
    InfMasking \cite{wen2025infmasking} & $67.15_{\pm 1.21}$ & $64.95_{\pm 0.92}$ & $66.92_{\pm 2.75}$ & $66.34$ \\
    SynCo (ours) + InfMasking & $\mathbf{67.52}_{\pm 1.11}$ & $\mathbf{66.18}_{\pm 0.57}$ & $\mathbf{67.48}_{\pm 1.48}$ & $\mathbf{67.06}$ \\
    \bottomrule
  \end{tabular}
  \label{tab:trimodal}
\end{table}
Table~\ref{tab:trimodal} reports the results. SynCo achieves the strongest performance across all three datasets and on average, outperforming CoMM by $+2.94\%$ and InfMasking by $+1.35\%$ in average classification accuracy. The improvement over CoMM is consistent with the bimodal results and indicates that the residual decomposition continues to provide effective synergy-focused supervision as the number of modalities increases. Combining SynCo with InfMasking improves over standalone InfMasking on average but does not exceed standalone SynCo in the trimodal setting.

\subsection{Ablation Experiment: Sensitivity to Loss Weights}
\label{app:loss_weights}

The two coefficients in the SynCo objective control disjoint parameter sets due to the stop-gradient configuration: $\lambda_{\text{pred}}$ governs only the optimization of $P$, while $\lambda_{\text{int}}$ scales the encoder gradient from $\mathcal{L}_{\text{interaction}}$. We sweep each coefficient on Trifeature with the other held at its default value.

\begin{figure}[ht]
    \centering
    \begin{subfigure}[b]{0.4\linewidth}
        \centering
        \includegraphics[width=\linewidth]{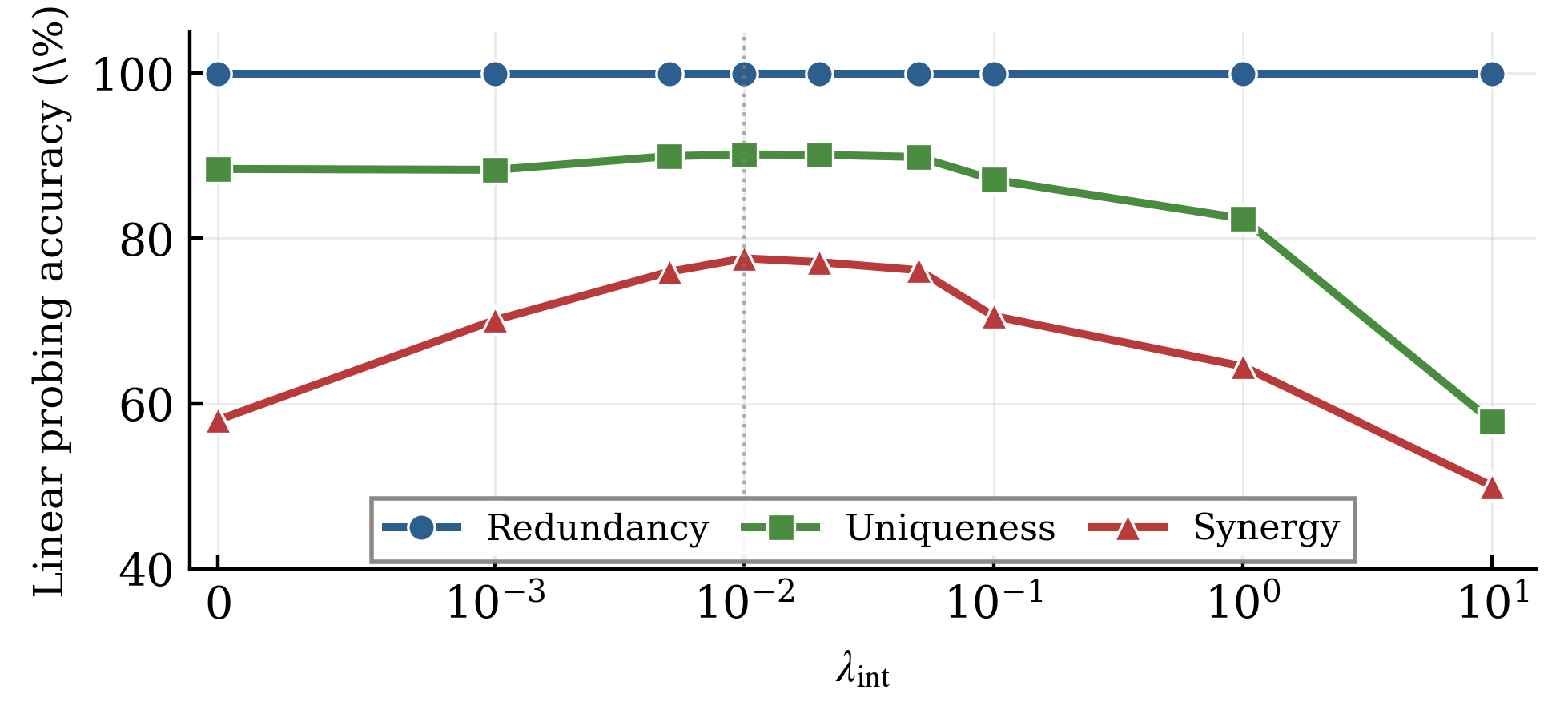}
        \phantomcaption
        \label{fig:lambda_int}
        {\small (a) $\lambda_{\mathrm{int}}$ sweep}
    \end{subfigure}
    \hspace{1em}
    \begin{subfigure}[b]{0.4\linewidth}
        \centering
        \includegraphics[width=\linewidth]{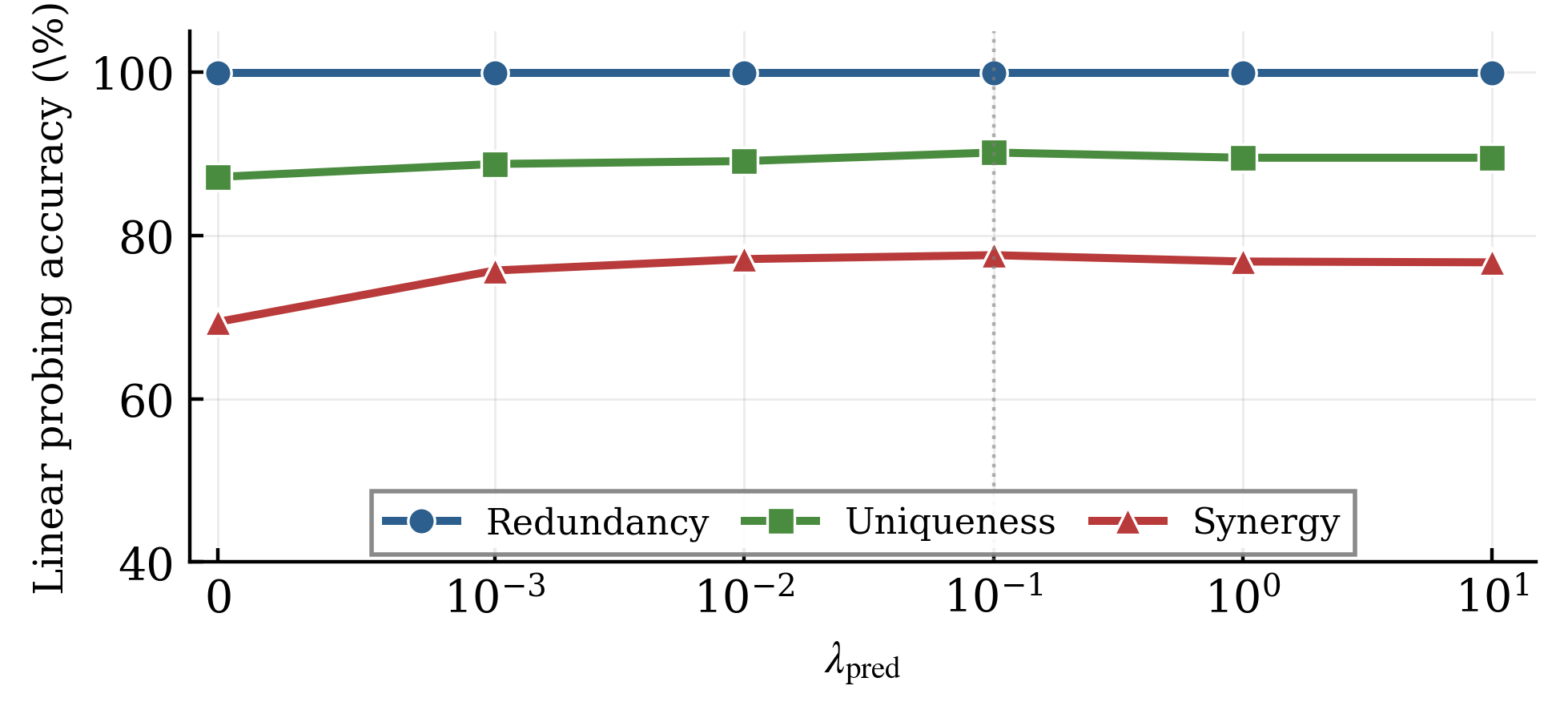}
        \phantomcaption
        \label{fig:lambda_pred}
        {\small (b) $\lambda_{\mathrm{pred}}$ sweep}
    \end{subfigure}
    \caption{Sensitivity of SynCo's loss weights on Trifeature. (a) $\lambda_{\mathrm{pred}} = 0.1$ fixed. (b) $\lambda_{\mathrm{int}} = 0.01$ fixed. Dotted lines mark the values used for all experiments.}
    \label{fig:ablation}
\end{figure}

Figure~\ref{fig:lambda_int} reports the sweep over $\lambda_{\text{int}}$ with $\lambda_{\text{pred}} = 0.1$. Synergy rises with $\lambda_{\text{int}}$, reaches a peak at $\lambda_{\text{int}} = 0.01$, and degrades at larger values. Redundancy remains stable across all values, while uniqueness degrades alongside synergy at very large $\lambda_{\text{int}}$, which indicates that overly large values disrupt overall representation quality. Figure~\ref{fig:lambda_pred} reports the sweep over $\lambda_{\text{pred}}$ with $\lambda_{\text{int}} = 0.01$. Once $\lambda_{\text{pred}}$ is sufficiently large to train $P$ effectively, all three components remain approximately flat across several orders of magnitude. The flat response is consistent with the gradient isolation property: $\lambda_{\text{pred}}$ scales the loss for $P$ alone and does not directly affect the encoder. The two sweeps show that $\lambda_{\text{int}} = 0.01$ yields the best synergy performance and $\lambda_{\text{pred}} = 0.1$ provides a sufficient threshold for training $P$; both are used as fixed defaults throughout all experiments.

\subsection{Ablation Experiment: Loss Terms}
\label{app:loss_ablation}

The SynCo objective in Eq.~\eqref{eq:synco_total} adds two terms to $\mathcal{L}_{\text{MCL}}$, and the interaction residual is defined relative to the fused representation that the multimodal term of Eq.~\eqref{eq:comm} shapes. To isolate the contribution of each component, we train three objectives on Trifeature under identical settings: $\mathcal{L}_{\text{MCL}}$ alone, the full SynCo objective, and the SynCo objective without the multimodal term $\hat{I}_{\text{NCE}}(Z', Z'')$. All other hyperparameters are held fixed at the values reported in Appendix~\ref{app:implementation_details}, and we report mean and standard deviation across five independent runs.

\begin{table}[ht]
  \centering
  \footnotesize
  \setlength{\tabcolsep}{6pt}
  \renewcommand{\arraystretch}{0.9}
  \caption{Linear probing accuracy (\%) of redundancy, uniqueness, and synergy on Trifeature under different combinations of loss terms.}
  \label{tab:loss_ablation}
  \begin{tabular}{@{}lccc@{}}
    \toprule
    \textit{Objective} & \textit{redundancy}$\uparrow$ & \textit{uniqueness}$\uparrow$ & \textit{synergy}$\uparrow$ \\
    \midrule
    $\mathcal{L}_{\text{MCL}}$ (CoMM) & $\mathbf{99.9}_{\pm0.05}$ & $86.6_{\pm1.32}$ & $71.6_{\pm1.07}$ \\
    SynCo w/o $\hat{I}_{\text{NCE}}(Z', Z'')$ & $99.7_{\pm0.37}$ & $88.1_{\pm0.49}$ & $75.5_{\pm0.28}$ \\
    SynCo (full) & $\mathbf{99.9}_{\pm0.21}$ & $\mathbf{90.15}_{\pm1.14}$ & $\mathbf{77.58}_{\pm0.74}$ \\
    \bottomrule
  \end{tabular}
\end{table}

Table~\ref{tab:loss_ablation} reports the results. The residual term accounts for most of the gain over the baseline: without the multimodal term, SynCo still improves synergy by $3.90$ points over $\mathcal{L}_{\text{MCL}}$, and adding the multimodal term recovers the remaining $2.08$ points. The multimodal term therefore remains necessary. It drives $Z$ toward $\Delta = 0$ in Theorem~\ref{thm:residual}, and the residual is informative about synergy to the extent that the fused representation it is computed from preserves the task-relevant information. Uniqueness follows the same pattern, while redundancy is unaffected by either term.

\subsection{Linear Probes on Unimodal Representations}
\label{app:unimodal_probes}

Hypothesis~\ref{hyp:synergy} concerns the linear prediction of the fused representation from the unimodal representations, and Table~\ref{tab:decomp} is consistent with it for the MSE-optimal projector $P^\star$. Here we probe the unimodal representations directly. We freeze the SynCo encoder trained on the Trifeature synergy task and train linear probes on the individual unimodal representations $Z_1$ and $Z_2$, on their concatenation $[Z_1; Z_2]$, and on the fused representation $Z$, following the protocol of Appendix~\ref{app:implementation_details}.

\begin{table}[ht]
  \centering
  \footnotesize
  \setlength{\tabcolsep}{6pt}
  \renewcommand{\arraystretch}{0.9}
  \caption{Linear probing accuracy (\%) on the Trifeature synergy task for the unimodal representations, their concatenation, and the fused representation. Chance level is 50\%.}
  \label{tab:unimodal_probes}
  \begin{tabular}{@{}lc@{}}
    \toprule
    \textit{Representation} & \textit{synergy}$\uparrow$ \\
    \midrule
    $Z_1$ (modality 1) & $50.00_{\pm0.00}$ \\
    $Z_2$ (modality 2) & $50.00_{\pm0.00}$ \\
    $[Z_1; Z_2]$ (concatenation) & $50.00_{\pm0.00}$ \\
    $Z$ (fused) & $\mathbf{77.58}_{\pm0.74}$ \\
    \bottomrule
  \end{tabular}
\end{table}

Table~\ref{tab:unimodal_probes} reports the results. Both unimodal representations and their concatenation remain at chance level, while the fused representation reaches $77.58\%$. The gap does not come from missing attribute information: texture and color are individually recoverable from the corresponding unimodal representation, and the synergy label depends on whether the particular texture-color pairing belongs to $\mathcal{M}$. Recovering the label therefore requires a function that combines the two modalities, which no linear function of $[Z_1; Z_2]$ provides. Since the linear prediction $\hat{Z} = P([Z_1; Z_2])$ is one such function, the concatenation result supports the premise of Hypothesis~\ref{hyp:synergy} independently of how well $P$ is trained.

\subsection{Ablation Experiment: Warmup Epoch Ablation}
\label{appendix:warmup_ablation}

The interaction weight $\lambda_{\text{int}}$ is linearly annealed from zero to its full value over a fixed number of warmup epochs at the start of training, as described in Appendix~\ref{app:implementation_details}. Early in training the linear projector $P$ is randomly initialized, so the residual $Z_{\text{res}} = Z - \hat{Z}$ initially approximates the fused representation itself rather than the component that excludes linearly unimodal-predictable content. The warmup delays the residual loss until $P$ has had time to approach its MSE-optimal solution. To assess the sensitivity of SynCo to this schedule, we sweep \texttt{warmup\_epochs} on Trifeature while holding all other hyperparameters fixed at the values reported in Appendix~\ref{app:implementation_details}.

\begin{figure}[ht]
    \centering
    \includegraphics[width=0.5\linewidth]{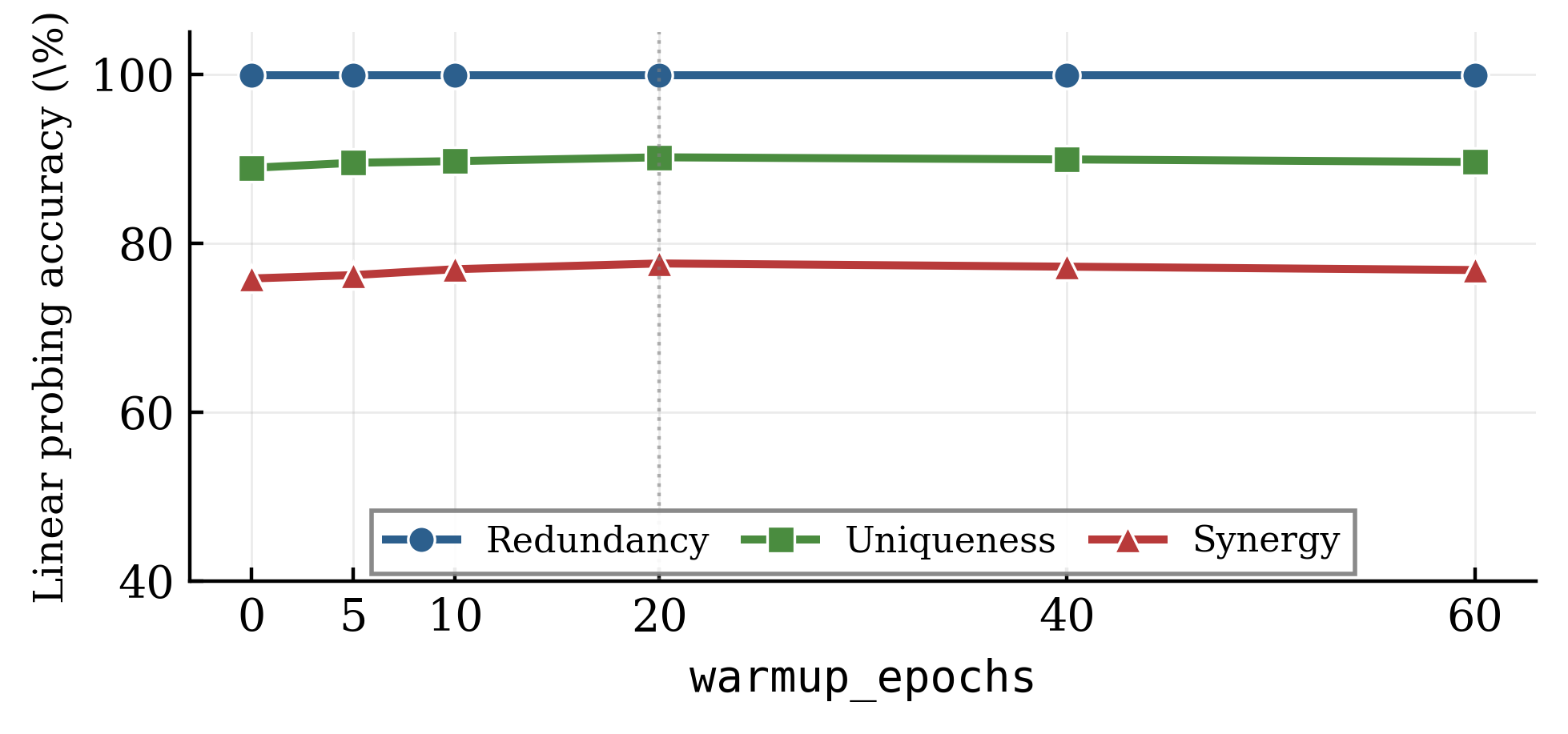}
    \caption{Sensitivity of SynCo to the warmup length on Trifeature. We sweep \texttt{warmup\_epochs} $\in \{0, 5, 10, 20, 40, 60\}$ with all other hyperparameters fixed. Linear probing accuracy (\%) is reported for each PID component, averaged over five independent runs. The dotted line marks the value used for all main-text experiments.}
    \label{fig:warmup_sweep}
\end{figure}

Figure~\ref{fig:warmup_sweep} reports the results. All three PID components remain stable across the sweep, with redundancy unaffected and uniqueness and synergy varying within roughly two points across the full range. Performance peaks at 20 epochs and declines mildly at the endpoints. We use 20 as the default across all experiments.

\section{Appendix Processing Times and Complexity}
\label{app:computational_requirements}

We benchmark the computational cost of SynCo against CoMM and InfMasking to characterize the overhead introduced by the residual decomposition relative to existing multimodal contrastive frameworks. All measurements are obtained on a single GPU with batch size 64, averaged over multiple training steps. Table~\ref{tab:computational} reports per-step time, peak memory, and parameter counts for each method.

\begin{table}[ht]
  \centering
  \footnotesize
  \setlength{\tabcolsep}{6pt}
  \renewcommand{\arraystretch}{0.9}
  \caption{Computational requirements of SynCo, CoMM, and InfMasking. Overhead is reported relative to CoMM. All measurements use batch size 64 on a single GPU.}
  \label{tab:computational}
  \begin{tabular}{@{}lcccc@{}}
    \toprule
    \textit{Method} & \textit{Time/step (ms)} & \textit{Memory (GB)} & \textit{Parameters} & \textit{Overhead} \\
    \midrule
    CoMM & $21.91 \pm 0.23$ & $0.17$ & $448{,}464$ & baseline \\
    SynCo (ours) & $21.94 \pm 0.18$ & $0.17$ & $451{,}704$ & $+0.15\%$ \\
    InfMasking & $57.82 \pm 0.33$ & $0.71$ & $448{,}504$ & $+163.92\%$ \\
    \bottomrule
  \end{tabular}
\end{table}

The results confirm that SynCo's residual decomposition introduces negligible computational cost. Relative to CoMM, SynCo increases per-step time by only $0.15\%$ and adds $3{,}240$ parameters, corresponding to the linear projector that maps concatenated unimodal features to the fused representation space. Memory consumption is essentially unchanged. By contrast, InfMasking incurs substantially higher cost than both CoMM and SynCo, with per-step time more than doubling and peak memory increasing by approximately a factor of four. This overhead reflects the masking strategy of InfMasking, which requires multiple forward passes over masked feature combinations during training. SynCo achieves comparable or superior synergy capture to InfMasking on the Trifeature benchmark and competitive performance across MultiBench and MM-IMDb at a fraction of the computational cost, making it the more efficient choice when training resources are constrained. The overhead of SynCo is bounded by the size of the linear projector, which scales as $O(nD^2)$ in the number of modalities $n$ and the fusion dimension $D$.

\section{Appendix Proofs}
\label{appendix:proofs}

This appendix provides detailed proofs of the theoretical results stated in the main text. Assumption~\ref{ass:augmentation} and Lemmas~\ref{lem:multimodal_mi} and~\ref{lem:unimodal_mi} follow CoMM~\citep{dufumier2024align}, and we prove the two lemmas in the notation of Section~\ref{sec:preliminaries}. Hypothesis~\ref{hyp:synergy} is assessed empirically in Section~\ref{sec:trifeature} and is treated as a premise rather than proven here. Throughout, the augmentations $t$, $t'$, and $t''$ are drawn from $\mathcal{T}$ independently of each other and of $(X, Y)$, so each augmented view depends on $(X, Y)$ only through $X$ and the chains $Y \to X \to X'$ and $Y \to X \to X''$ are Markov. We write sums for discrete variables; the arguments are identical with integrals for continuous variables.

\subsection*{Equivalent Form of Assumption~\ref{ass:augmentation}}

\begin{lemma}
\label{lem:ass_equiv}
Let $X' = t(X)$ with $t \in \mathcal{T}$. The equalities $I(X; X') = I(X; Y)$ and $I(X'; Y) = I(X; Y)$ of Assumption~\ref{ass:augmentation} hold if and only if $X' \perp X \mid Y$ and $Y \perp X \mid X'$.
\end{lemma}

\begin{proof}
\textit{Step 1: Label preservation.}
The Markov chain $Y \to X \to X'$ gives $I(X'; Y \mid X) = 0$, and the chain rule expands $I(X, X'; Y)$ in two orders as
\begin{equation}
I(X; Y) \;=\; I(X, X'; Y) \;=\; I(X'; Y) + I(X; Y \mid X').
\label{eq:appx_ass_1}
\end{equation}
The equality $I(X'; Y) = I(X; Y)$ therefore holds if and only if $I(X; Y \mid X') = 0$, that is, $Y \perp X \mid X'$.

\textit{Step 2: Minimality.}
The chain rule expands $I(X; X', Y)$ in two orders as
\begin{equation}
I(X; X') + I(X; Y \mid X') \;=\; I(X; Y) + I(X; X' \mid Y).
\label{eq:appx_ass_2}
\end{equation}
When $I(X; Y \mid X') = 0$, Eq.~\eqref{eq:appx_ass_2} reduces to $I(X; X') - I(X; Y) = I(X; X' \mid Y)$, and the equality $I(X; X') = I(X; Y)$ holds if and only if $I(X; X' \mid Y) = 0$, that is, $X' \perp X \mid Y$.
\end{proof}

The condition $X' \perp X \mid Y$ states that the augmented view carries no information about $X$ beyond the label, and the condition $Y \perp X \mid X'$ states that the view retains all information that $X$ carries about the label.

\subsection*{Proof of Lemma~\ref{lem:multimodal_mi}}

\begin{proof}
Let $Z_\theta = f_\theta(X)$ and $Z'_\theta = f_\theta(X')$ with $X' = t(X)$ for some $t \in \mathcal{T}$.

\textit{Step 1: Upper bounds.}
The pair $(Z_\theta, Z'_\theta)$ is a deterministic function of $(X, X')$, and the data processing inequality together with Assumption~\ref{ass:augmentation} gives
\begin{equation}
I(Z_\theta; Z'_\theta) \;\leq\; I(X; X') \;=\; I(X; Y).
\label{eq:appx_lem1_1}
\end{equation}
By Lemma~\ref{lem:ass_equiv}, $X' \perp X \mid Y$. Since $Z_\theta$ and $Z'_\theta$ are functions of $X$ and $X'$ respectively, $Z'_\theta \perp Z_\theta \mid Y$, and the data processing inequality applied to the chains $Z_\theta \to Y \to Z'_\theta$ and $Y \to X \to Z_\theta$ gives
\begin{equation}
I(Z_\theta; Z'_\theta) \;\leq\; I(Z_\theta; Y) \;\leq\; I(X; Y).
\label{eq:appx_lem1_2}
\end{equation}

\textit{Step 2: Tightness under sufficient expressivity.}
Under sufficient expressivity, the encoder family $\{f_\theta\}$ contains a map that is injective on the supports of $X$ and $X'$, such as the identity, and such a map attains the bound in Eq.~\eqref{eq:appx_lem1_1}. Hence
\begin{equation}
I(Z_{\theta^\star}; Z'_{\theta^\star}) \;=\; \max_{\theta} I(Z_\theta; Z'_\theta) \;=\; I(X; X') \;=\; I(X; Y).
\label{eq:appx_lem1_3}
\end{equation}
Substituting Eq.~\eqref{eq:appx_lem1_3} into Eq.~\eqref{eq:appx_lem1_2} gives $I(X; Y) \leq I(Z_{\theta^\star}; Y) \leq I(X; Y)$, and therefore $I(Z_{\theta^\star}; Y) = I(X; Y)$.
\end{proof}

The first term of Eq.~\eqref{eq:comm} contrasts two augmented views rather than an input and its augmentation. The following lemma shows that the pair $(X', X'')$ satisfies the conditions used in the proof above.

\begin{lemma}
\label{lem:two_views}
Let $X' = t'(X)$ and $X'' = t''(X)$ with $t', t'' \in \mathcal{T}$. Under Assumption~\ref{ass:augmentation}, $X' \perp X'' \mid Y$ and $I(X'; X'') = I(X'; Y) = I(X''; Y) = I(X; Y)$.
\end{lemma}

\begin{proof}
Given $X$, the views $X'$ and $X''$ are independent of each other and of $Y$. By Lemma~\ref{lem:ass_equiv}, $X' \perp X \mid Y$, which together with the Markov chain $Y \to X \to X'$ gives $p(x' \mid x) = p(x' \mid y)$ for every $(x, y)$ in the support. Therefore
\begin{equation}
p(x', x'' \mid y) \;=\; \sum_{x} p(x \mid y)\, p(x' \mid x)\, p(x'' \mid x) \;=\; p(x' \mid y)\, p(x'' \mid y),
\end{equation}
which proves $X' \perp X'' \mid Y$. Lemma~\ref{lem:ass_equiv} applied to $X''$ gives $Y \perp X \mid X''$, and therefore
\begin{equation}
p(y, x' \mid x'') \;=\; \sum_{x} p(x \mid x'')\, p(y \mid x'')\, p(x' \mid x) \;=\; p(y \mid x'')\, p(x' \mid x''),
\end{equation}
which proves $I(X'; Y \mid X'') = 0$. Expanding $I(X'; X'', Y)$ in two orders, as in Eq.~\eqref{eq:appx_ass_2}, gives
\begin{equation}
I(X'; X'') \;=\; I(X'; Y) + I(X'; X'' \mid Y) - I(X'; Y \mid X'') \;=\; I(X'; Y),
\end{equation}
and Assumption~\ref{ass:augmentation} gives $I(X'; Y) = I(X''; Y) = I(X; Y)$.
\end{proof}

With Lemma~\ref{lem:two_views} in place of Lemma~\ref{lem:ass_equiv}, the proof of Lemma~\ref{lem:multimodal_mi} applies to $Z'_\theta$ and $Z''_\theta$. The optimal parameters that maximize $I(Z'_\theta; Z''_\theta)$ therefore satisfy $I(Z'_{\theta^\star}; Z''_{\theta^\star}) = I(Z'_{\theta^\star}; Y) = I(Z''_{\theta^\star}; Y) = I(X; Y) = R + S + U_1 + U_2$, which is the quantity targeted by the first term of Eq.~\eqref{eq:comm}.

\subsection*{Proof of Lemma~\ref{lem:unimodal_mi}}

\begin{proof}
\textit{Step 1: Assumption~\ref{ass:augmentation} for a single modality.}
The selected modality $X_i = t_i(X)$ is a deterministic function of $X$. By Lemma~\ref{lem:ass_equiv}, $X' \perp X \mid Y$ and $Y \perp X \mid X'$, which imply $X' \perp X_i \mid Y$ and $Y \perp X_i \mid X'$. Expanding $I(X_i; X', Y)$ in two orders, as in Eq.~\eqref{eq:appx_ass_2}, gives
\begin{equation}
I(X_i; X') \;=\; I(X_i; Y) + I(X_i; X' \mid Y) - I(X_i; Y \mid X') \;=\; I(X_i; Y) \;=\; R + U_i,
\label{eq:appx_lem2_1}
\end{equation}
where the last equality follows from the consistency equations in Eq.~\eqref{eq:consistency}.

\textit{Step 2: Upper bounds.}
The unimodal representation $Z_i = f_\theta(t_i(X))$ is a function of $X_i$, and the data processing inequality together with Eq.~\eqref{eq:appx_lem2_1} gives
\begin{equation}
I(Z_i; Z') \;\leq\; I(X_i; X') \;=\; R + U_i.
\label{eq:appx_lem2_2}
\end{equation}
Since $X' \perp X_i \mid Y$, the representations satisfy $Z' \perp Z_i \mid Y$, and the data processing inequality applied to the chains $Z_i \to Y \to Z'$ and $Y \to X_i \to Z_i$ gives
\begin{equation}
I(Z_i; Z') \;\leq\; I(Z_i; Y) \;\leq\; I(X_i; Y) \;=\; R + U_i.
\label{eq:appx_lem2_3}
\end{equation}

\textit{Step 3: Tightness under sufficient expressivity.}
As in the proof of Lemma~\ref{lem:multimodal_mi}, a map that is injective on the supports of $X_i$ and $X'$ attains the bound in Eq.~\eqref{eq:appx_lem2_2}, and the optimal parameters $\theta^\star$ that maximize $I(Z_i; Z')$ satisfy $I(Z_i; Z') = R + U_i$. Eq.~\eqref{eq:appx_lem2_3} then holds with equality:
\begin{equation}
I(Z_i; Z') \;=\; I(Z_i; Y) \;=\; I(X_i; Y) \;=\; R + U_i.
\end{equation}
The argument with $Z''$ in place of $Z'$ is identical.
\end{proof}

\paragraph{Relation to CoMM.} CoMM states the unimodal result as the special case of its multimodal lemma in which the augmentation is the projection $t_i$. Under Assumption~\ref{ass:augmentation}, $t_i \in \mathcal{T}$ would require $I(X_i; Y) = I(X; Y)$, which by Eqs.~\eqref{eq:pid} and~\eqref{eq:consistency} holds only when $S = 0$ and $U_j = 0$ for $j \neq i$. The proof above instead transfers Assumption~\ref{ass:augmentation} from $X$ to $X_i$ in Step 1 and does not require $t_i \in \mathcal{T}$.

\subsection*{Proof of Lemma~\ref{lem:info_decomp}}

\begin{proof}
Recall $\hat{Z} = P([Z_1; Z_2])$ and $Z_\mathrm{res} = Z - \hat{Z}$, so $Z = \hat{Z} + Z_\mathrm{res}$ pointwise.

\textit{Step 1: Equivalence under conditioning on $\hat{Z}$.}
Given $\hat{Z}$, the residual $Z_\mathrm{res}$ determines $Z$ via $Z = \hat{Z} + Z_\mathrm{res}$, and conversely $Z$ determines $Z_\mathrm{res}$ via $Z_\mathrm{res} = Z - \hat{Z}$. The two random variables therefore carry the same information given $\hat{Z}$, hence
\begin{equation}
I(Z_\mathrm{res}; Y \mid \hat{Z}) \;=\; I(Z; Y \mid \hat{Z}).
\label{eq:appx_lem3_1}
\end{equation}

\textit{Step 2: Chain rule decomposition.}
The chain rule for mutual information gives
\begin{equation}
I(Z, \hat{Z}; Y) \;=\; I(\hat{Z}; Y) + I(Z; Y \mid \hat{Z}).
\label{eq:appx_lem3_2}
\end{equation}

\textit{Step 3: Bounds on $I(Z, \hat{Z}; Y)$.}
Both $Z = f_\theta(X)$ and $\hat{Z} = P([Z_1; Z_2]) = P([f_\theta(t_1(X)); f_\theta(t_2(X))])$ are deterministic functions of $X$, so the joint $(Z, \hat{Z})$ is a function of $X$. The data processing inequality therefore yields
\begin{equation}
I(Z, \hat{Z}; Y) \;\leq\; I(X; Y).
\label{eq:appx_lem3_3}
\end{equation}
The reverse direction follows from the chain rule, $I(Z, \hat{Z}; Y) = I(Z; Y) + I(\hat{Z}; Y \mid Z)$, and the non-negativity of conditional mutual information:
\begin{equation}
I(Z, \hat{Z}; Y) \;\geq\; I(Z; Y).
\label{eq:appx_lem3_4}
\end{equation}
When $I(Z; Y) = I(X; Y)$, Eqs.~\eqref{eq:appx_lem3_3}--\eqref{eq:appx_lem3_4} give
\begin{equation}
I(Z, \hat{Z}; Y) \;=\; I(Z; Y).
\label{eq:appx_lem3_5}
\end{equation}

\textit{Step 4: Combine.}
Combining Eq.~\eqref{eq:appx_lem3_4} with Eq.~\eqref{eq:appx_lem3_2} and rearranging,
\begin{equation}
I(Z; Y \mid \hat{Z}) \;\geq\; I(Z; Y) - I(\hat{Z}; Y),
\end{equation}
with equality under Eq.~\eqref{eq:appx_lem3_5}. By Eq.~\eqref{eq:appx_lem3_1}, the left side equals $I(Z_\mathrm{res}; Y \mid \hat{Z})$, completing the proof.
\end{proof}

\subsection*{Proof of Theorem~\ref{thm:residual}}

\begin{proof}
Apply Lemma~\ref{lem:info_decomp} to the MSE-optimal projector $P^\star$ with corresponding prediction $\hat{Z}^\star$ and residual $Z^\star_\mathrm{res} = Z - \hat{Z}^\star$:
\begin{equation}
I(Z^\star_\mathrm{res}; Y \mid \hat{Z}^\star) \;\geq\; I(Z; Y) - I(\hat{Z}^\star; Y).
\label{eq:appx_thm1_1}
\end{equation}
The PID decomposition in Eq.~\eqref{eq:pid} and the definition of $\Delta$ give
\begin{equation}
I(Z; Y) \;=\; I(X; Y) - \Delta \;=\; R + S + U_1 + U_2 - \Delta.
\label{eq:appx_thm1_2}
\end{equation}
Hypothesis~\ref{hyp:synergy} bounds the task-relevant content of the linear prediction:
\begin{equation}
I(\hat{Z}^\star; Y) \;\leq\; R + U_1 + U_2.
\label{eq:appx_thm1_3}
\end{equation}
Substituting Eqs.~\eqref{eq:appx_thm1_2}--\eqref{eq:appx_thm1_3} into Eq.~\eqref{eq:appx_thm1_1},
\begin{equation}
I(Z^\star_\mathrm{res}; Y \mid \hat{Z}^\star) \;\geq\; (R + S + U_1 + U_2 - \Delta) - (R + U_1 + U_2) \;=\; S - \Delta.
\label{eq:appx_thm1_4}
\end{equation}
When $\Delta = 0$, Eq.~\eqref{eq:appx_thm1_4} gives $I(Z^\star_\mathrm{res}; Y \mid \hat{Z}^\star) \geq S$, so the residual provides at least $S$ additional task-relevant information given the linear prediction.
\end{proof}

\subsection*{Proof of Corollary~\ref{cor:synergy_contrastive}}

\begin{proof}
Let $X' = t'(X)$ and $X'' = t''(X)$ for $t', t'' \in \mathcal{T}$, with corresponding fused representations $Z' = f_{\theta^\star}(X')$ and $Z'' = f_{\theta^\star}(X'')$, unimodal representations $Z'_i = f_{\theta^\star}(t_i(X'))$ and $Z''_i = f_{\theta^\star}(t_i(X''))$, and residuals
\begin{equation}
Z'_\mathrm{res} \;=\; Z' - P^\star([Z'_1; Z'_2]), \qquad Z''_\mathrm{res} \;=\; Z'' - P^\star([Z''_1; Z''_2]).
\end{equation}

\textit{Step 1: Task-relevant information is preserved across views.}
By Assumption~\ref{ass:augmentation}, $I(X'; Y) = I(X''; Y) = I(X; Y)$, so the total task-relevant information $R + S + U_1 + U_2$ is preserved in both views. By Lemma~\ref{lem:two_views}, the fused representations satisfy $I(Z'; Y) = I(Z''; Y) = I(X; Y)$ at the information-preserving optimum, which gives $\Delta = 0$ in Theorem~\ref{thm:residual} for each view.

\textit{Step 2: Conditional task-relevant information in the residuals.}
By Theorem~\ref{thm:residual} applied to each view,
\begin{equation}
I(Z'_\mathrm{res}; Y \mid \hat{Z}'^\star) \;\geq\; S, \qquad I(Z''_\mathrm{res}; Y \mid \hat{Z}''^\star) \;\geq\; S,
\end{equation}
where $\hat{Z}'^\star$ and $\hat{Z}''^\star$ are the corresponding linear predictions. Given the linear prediction, each residual therefore provides at least $S$ additional task-relevant information.

\textit{Step 3: InfoNCE rewards task-relevant information shared between the residuals.}
The InfoNCE estimator $\hat{I}_{\mathrm{NCE}}(Z'_\mathrm{res}, Z''_\mathrm{res})$ is a lower bound on $I(Z'_\mathrm{res}; Z''_\mathrm{res})$, and maximizing this lower bound rewards information shared between the two residuals. Each residual is a deterministic function of its corresponding augmented view, and $X' \perp X'' \mid Y$ by Lemma~\ref{lem:two_views}. Consequently, $Z'_\mathrm{res} \perp Z''_\mathrm{res} \mid Y$, and the data processing inequality gives
\begin{equation}
I(Z'_\mathrm{res}; Z''_\mathrm{res})
\;\leq\;
\min\left\{
I(Z'_\mathrm{res}; Y),
I(Z''_\mathrm{res}; Y)
\right\}.
\end{equation}
The information shared between the interaction residuals is therefore task-relevant under Assumption~\ref{ass:augmentation}. By Step 2, each residual provides at least $S$ additional task-relevant information given its linear prediction. Applying InfoNCE to the interaction residuals encourages the encoder to preserve shared task-relevant information in the component remaining after subtraction of the linear prediction, providing the motivation for synergy-focused supervision.
\end{proof}

\section{Appendix Discussion of the Augmentation Assumption}
\label{app:assumption}

We discuss the interpretation and feasibility of Assumption~\ref{ass:augmentation}, its role in SynCo, and the effect of an approximate version on the guarantee of Theorem~\ref{thm:residual}.

\paragraph{Interpretation.} Lemma~\ref{lem:ass_equiv} restates Assumption~\ref{ass:augmentation} as two conditional independences. The condition $X' \perp X \mid Y$ requires the augmented view to carry no information about $X$ beyond the label, and the condition $Y \perp X \mid X'$ requires the view to retain all information that $X$ carries about the label. Practical augmentation families satisfy both conditions only approximately. The assumption formalizes the goal of view design in contrastive learning, where augmentations remove nuisance information and preserve task-relevant information~\citep{tian2020makes}, and extends the goal to multimodal augmentations that are jointly defined on $(X_1, X_2)$.

\paragraph{Feasibility.} CoMM examines the feasibility of the assumption by varying the augmentation strength on Trifeature and MM-IMDb~\citep[Appendix~C.4]{dufumier2024align}. Downstream performance improves as stronger augmentations remove nuisance information and declines once the augmentations also remove task-relevant information. The interior optimum corresponds to the regime that Assumption~\ref{ass:augmentation} describes.

\paragraph{Role in SynCo.} Assumption~\ref{ass:augmentation} underlies Lemmas~\ref{lem:multimodal_mi} and~\ref{lem:unimodal_mi}, which characterize the optimum of $\mathcal{L}_{\text{MCL}}$ in Eq.~\eqref{eq:comm}. CoMM, InfMasking~\citep{wen2025infmasking}, and SynCo all train with $\mathcal{L}_{\text{MCL}}$, so the assumption applies equally to all three methods and does not differentiate SynCo from them. Lemma~\ref{lem:info_decomp} and the bound of Theorem~\ref{thm:residual} do not require Assumption~\ref{ass:augmentation}; the assumption enters only through $\Delta = 0$ in Theorem~\ref{thm:residual}. The premise specific to SynCo is Hypothesis~\ref{hyp:synergy}, which concerns the trained encoder and the linear projector rather than the augmentation family and is assessed empirically in Table~\ref{tab:decomp}.

\paragraph{Approximate label preservation.} When the augmentations preserve task-relevant information only approximately, $\Delta$ can be positive, and Theorem~\ref{thm:residual} still gives $I(Z^\star_\mathrm{res}; Y \mid \hat{Z}^\star) \geq S - \Delta$.

\section{Appendix Broader Impact}
\label{app:broader_impact}

SynCo is a method for self-supervised multimodal representation learning that improves the capture of synergistic information across modalities. As foundational research on contrastive learning objectives, the work is not tied to a specific application or deployment, but its downstream applications span both beneficial and potentially sensitive domains.

\paragraph{Potential positive impacts.} By improving the quality of representations learned from unlabeled multimodal data, SynCo can reduce the labeling burden in domains where annotations are expensive or require expert knowledge. The benchmarks evaluated in this paper highlight several such domains: clinical decision support from physiological time-series and tabular records (MIMIC), human activity recognition from wearable sensors (DARai), and robotic manipulation from vision and force-torque feedback (Vision\&Touch). Improved synergy capture is particularly valuable in these settings, where the most informative cues often emerge only from the joint observation of multiple modalities. SynCo also introduces negligible computational overhead relative to existing contrastive frameworks (Appendix~\ref{app:computational_requirements}), lowering the energy cost and hardware barrier to applying multimodal self-supervised learning compared to more compute-intensive alternatives.

\paragraph{Potential negative impacts.} As with any general-purpose representation learning method, SynCo could in principle be applied in settings with negative consequences, such as surveillance applications. The method is presented as a general-purpose training objective rather than a deployable system, and we do not release any pretrained model that constitutes a direct path to such applications.

\paragraph{Fairness and reliability.} Models trained with SynCo inherit the properties of their training data and backbone encoders. In high-stakes domains, deployment should follow established validation and evaluation protocols, which are independent of the training objective.

\section{Appendix Dataset Details}
\label{sec:datasets}

\subsection{Trifeature}

The Trifeature benchmark~\citep{hermann2020shapes} was originally introduced for studying how convolutional networks represent visual factors of variation. Each image contains a single shape rendered with a chosen texture and a chosen color, and shape, texture, and color each take one of ten values, giving $1{,}000$ unique (shape, texture, color) combinations. We assign $800$ combinations to training and reserve the remaining $200$ for evaluation. Every training combination is rendered three times, with the shape and texture orientations sampled independently from $[-45^\circ, 45^\circ]$. The shape occupies a $128 \times 128$ region placed at a uniformly random location inside a $224 \times 224$ image canvas so that the region lies entirely within the canvas, and the texture and color are then applied. We form bimodal examples by pairing two such images, producing $10{,}000$ training pairs and $4{,}096$ test pairs from the same distribution. Section~\ref{sec:trifeature} describes how each pair is labeled to obtain the redundancy, uniqueness, and synergy probes.

\subsection{MultiBench}

We use five datasets from the MultiBench benchmark~\citep{liang2021multibench} and follow the preprocessing released with the benchmark in every case. All datasets are anonymized at source and contain no personally identifiable information.

\paragraph{MIMIC.} The MIMIC dataset~\citep{johnson2016mimic} comprises de-identified critical-care records gathered between 2001 and 2012, totaling $53{,}423$ admissions from $38{,}597$ distinct patients. We use the two modalities released by MultiBench: a time-series modality of hourly $12$-dimensional physiological measurements taken over the first $24$ hours of an admission, and a static modality of $5$-dimensional patient demographics such as age and gender. Consistent with~\citep{liang2023factorized}, the target is a binary indicator of whether an admission falls under ICD-9 code group~$7$ ($460$-$519$), corresponding to diseases of the respiratory system.

\paragraph{MOSI.} MOSI~\citep{zadeh2016multimodal} is a sentiment analysis benchmark built from $2{,}199$ short YouTube monologue clips, each accompanied by video frames, audio, and a transcript of the spoken content. The original labels are continuous sentiment scores in $[-3, 3]$, which we binarize into positive and negative classes consistent with~\citep{liang2023factorized}. Our bimodal experiments use the visual and textual modalities.

\paragraph{UR-FUNNY.} UR-FUNNY~\citep{hasan2019ur} is a benchmark for humor detection built from $1{,}866$ TED talks, providing $16{,}514$ short video segments with aligned audio and transcripts, each labeled as humorous or not. Our bimodal experiments use the visual and textual modalities.

\paragraph{MUSTARD.} MUSTARD~\citep{castro2019towards} is a sarcasm detection benchmark comprising utterances from popular television shows, including \textit{Friends}, \textit{The Big Bang Theory}, and \textit{The Golden Girls}. We use the balanced split of $690$ utterances, with every utterance providing synchronized video, audio, and transcript along with a binary sarcasm label. Our bimodal experiments use the visual and textual modalities.

\paragraph{Vision\&Touch.} Vision\&Touch~\citep{lee2020making} records $150$ trajectories of $1{,}000$ time steps each, captured from a 7-DoF Franka Panda robot performing a peg-insertion task. Every step provides an RGB image, a depth map, force readings, and the end-effector position and velocity. Consistent with the regression setting in MultiBench~\citep{liang2021multibench}, we predict the next-step end-effector position under mean-squared error. Our bimodal experiments use the visual and proprioceptive modalities.

\subsection{DARai}
DARai~\citep{kaviani2025hierarchical} is a multimodal benchmark for human activity recognition, providing over $200$ hours of synchronized recordings from $50$ participants in $10$ indoor environments. The full release contains $20$ sensor streams, including multiple camera views, depth and radar sensors, wearable inertial measurement units (IMU), forearm electromyography (EMG), instrumented insoles, biomonitors, and a gaze tracker. Annotations are organized in a three-level hierarchy spanning high-level activities, lower-level actions, and fine-grained procedures. We use DARai to assess whether the gains of SynCo extend beyond the vision-language regime that dominates the other benchmarks. The wearable subset operates on physical signals with markedly different temporal dynamics and information content from natural images and text.

Each wearable modality captures a distinct physical aspect of human activity. The wearable IMU records three-axis accelerometer, gyroscope, and magnetometer signals at $12$~Hz and reflects the kinematics of the upper body, including limb orientation, acceleration, and rotation. The forearm EMG records muscle activity at $4$~kHz and carries information about exerted effort that is largely invisible to a purely kinematic sensor, since the same motion can arise from very different muscular loads. The instrumented insoles record eight-point pressure at $500$~Hz and capture weight transfer, balance, and ground-contact patterns that the upper-body sensors cannot observe. The wearable subset thus pairs signals that differ in physical origin, body location, and temporal scale.

To assess whether the gains of SynCo depend on the choice of modality pairing, we evaluate three bimodal configurations: IMU + Insole, IMU + EMG, and EMG + Insole. The IMU + Insole pairing combines arm-mounted kinematics with foot-mounted pressure. Activities requiring coordinated upper- and lower-body motion, such as climbing stairs or carrying an object, can only be inferred by integrating both signals. The IMU + EMG pairing relates gross kinematics to muscular effort and exposes activities where the same motion arises from different effort profiles, such as lifting a heavy versus a light object. The EMG + Insole pairing combines forearm and foot signals from physically separated body regions, requiring the model to relate cues across distant sensors rather than rely on a single body segment. The three configurations share the same preprocessing and backbone networks, as reported in Table~\ref{tab:darai}.

\subsection{MM-IMDb}
MM-IMDb~\citep{arevalo2017gated} is a multilabel benchmark for movie genre classification, with $25{,}959$ films described by a poster image and a plot summary. The genre annotations span $23$ categories, and a single film often belongs to several of them, which makes the task multilabel rather than multiclass. We use the poster and plot modalities. Although MM-IMDb is included in MultiBench~\citep{liang2021multibench}, we report it separately because we train directly on the raw modalities rather than on the preprocessed feature vectors that MultiBench provides.


\end{document}